\ifdefined\pdfoutput\pdfoutput=1\fi
\documentclass{article}

\PassOptionsToPackage{numbers,compress}{natbib}

\usepackage[preprint]{neurips_2026}

\usepackage[utf8]{inputenc} %
\usepackage[T1]{fontenc}    %
\usepackage{hyperref}       %
\usepackage{url}            %
\usepackage{booktabs}       %
\usepackage{amsfonts}       %
\usepackage{nicefrac}       %
\usepackage{microtype}      %
\usepackage{xcolor}         %

\usepackage{amsmath,amssymb,amsthm,graphicx,enumitem,mathrsfs}
\usepackage{subcaption}

\usepackage{amsthm}
\usepackage{dsfont}
\theoremstyle{plain}
\newtheorem{theorem}{Theorem}[section]
\newtheorem{lemma}{Lemma}[section]

\newtheorem{remark}{Remark}[section]
\usepackage{tcolorbox}
\usepackage{booktabs}  
\usepackage{placeins}
\usepackage{bm} 
\theoremstyle{remark}

\newtheoremstyle{bolddefinition}%
  {3pt}%
  {3pt}%
  {\itshape}%
  {}%
  {\bfseries\boldmath}%
  {.}%
  {.5em}%
  {\thmname{#1}\thmnumber{ #2}\textbf{\thmnote{ [#3]}}}
\theoremstyle{bolddefinition}
\newtheorem{definition}{Definition}[section]
\newtheorem{assumption}[definition]{Assumption}
\newtheorem{proposition}[theorem]{Proposition}
\usepackage{algorithm}
\usepackage{algorithmic}
\usepackage{wrapfig}
 \usepackage{pifont}

\newcommand{\Inf}{\mathcal{I}}
\newcommand{\D}{\mathcal{D}}
\newcommand{\A}{\mathcal{A}}
\newcommand{\prior}{\mathcal{P}}
\newcommand{\prob}{\mathbb{P}}
\newcommand{\E}{\mathbb{E}}
\newcommand{\cpi}{\{\pi^i\}}
\newcommand{\bpi}{\{\bar{\pi}^i\}}
\newcommand{\lstm}{\text{LSTM}}
\newcommand{\emb}{\text{Embed}}
\newcommand{\yes}{\ding{51}}
\newcommand{\no}{\ding{55}}

\definecolor{citeblue}{RGB}{33,99,181}
\definecolor{refpink}{RGB}{214,51,132}

\definecolor{TD3_color}{HTML}{3C9BC9}
\definecolor{VND_color}{HTML}{A6CA68}
\definecolor{UNIFORM_color}{HTML}{FFB83D}
\definecolor{INDEP_color}{HTML}{FC757B}

\DeclareMathOperator*{\argmax}{arg\,max}

\hypersetup{
  colorlinks=true,
  citecolor=citeblue,
  linkcolor=refpink,
  urlcolor=black
}

\title{ICMAPE: In-Context Multiagent Pure Exploration}

\author{%
  Xinyi Hu \\
  Boston University \\
  \texttt{xhu07@bu.edu} \\
  \And
  Alessio Russo \\
  Boston University \\
  \texttt{arusso2@bu.edu} \\
  \And
  Aldo Pacchiano \\
  Boston University \\
  Broad Institute of MIT and Harvard \\
  \texttt{pacchian@bu.edu} \\
}

\begin{document}

\maketitle

\begin{abstract}
In some multi-agent systems, the quantity to be optimized is not an externally specified reward but the information acquired about unknown properties of the environment as done in active sequential hypothesis testing (ASHT) problems. However, the ASHT literature tends to focus on finite single-agent problems with well-specified models, while there is currently a gap for practical multi-agent methods that can perform active sequential testing. 
We fill this gap with ICMAPE, a Bayesian learning-based framework for decentralized multi-agent pure-exploration driven by inference objectives. ICMAPE converts the fixed-confidence identification objective into a reward derived from inference confidence, so that standard reinforcement learning machinery can be applied to decentralized pure exploration.  It jointly learns a centralized neural inference network that estimates a posterior distribution over hypotheses from global trajectory data, and decentralized policies that select actions from local observation histories and learn when to stop collecting data once the target confidence is reached.
On two synthetic benchmarks and a Maryland nitrate concentration monitoring task based on real-world data, ICMAPE-TD3 achieves target accuracy with fewer exploration steps.
Code is available at \url{https://github.com/xinyihu3721/ICMAPE}.
\end{abstract}

\section{Introduction}
The majority of existing multi-agent reinforcement learning (MARL) algorithms are designed to maximize cumulative reward \citep{ming1993, matignon2012,son2019,lowe2017,foerster2018,iqbal2019,ackermann2019,sheikh2020,zhang2024, wei2025lero}. 
 However, in some multi-agent systems the goal is to collaboratively collect informative data in order to identify unknown properties of the environment while minimizing data collection cost.
For example, in distributed sensor localization, sensors combine local measurements to infer an unknown target or source location \citep{patwari2005locating,MAO2007,jsan6040024}. 
We study an active setting in which agents share a common hypothesis identification objective but, due to practical constraints such as limited bandwidth, organizational boundaries, confidentiality requirements, or reporting delays, must adaptively select what information to collect using only their local observation histories, while a centralized inference module aggregates the collected data to make the final inference. Such a perspective is closely related to active sequential hypothesis testing and pure exploration, where actions are selected adaptively to identify the true hypotheses with minimal sampling cost \citep{naghshvar2013Active, kaufmann2016,garivier2016,russo2016Simple,russo2025pure,vannella2024}. Even though multi-agent and federated variants exist, they often retain centralized action selection or focus primarily on communication constraints rather than decentralized data collection~\cite{vannella2023,mitra2021, kota2023, Chen_2024}.

Decentralized multi-agent pure exploration poses three challenges. First, each agent must make decisions from only its local trajectory, without observing the full joint exploration history. Second, because the hypothesis is identified from the collective data, the decentralized policies must induce coordinated exploration patterns that are informative at the global level rather than redundant. Third, agents must learn when to stop, since sampling past identifiability only adds additional cost.

Standard cooperative MARL approaches under centralized training with decentralized execution (CTDE) are typically formulated for maximizing cumulative environment reward ~\citep{sunehag2017, rashid2020, lowe2017, foerster2018, ackermann2019, zhang2024}, so the high reward behaviors they learn need not be informative about the underlying hypothesis. Pure exploration methods optimize the right objective, but are tailored to specific bandit models and do not transfer to decentralized multi-agent inference problems \citep{kaufmann2016, garivier2016}. Neither line therefore provides a general learning-based method for decentralized data collection driven by an inference objective.

We propose \emph{In-Context Multi-Agent Pure Exploration} (ICMAPE), a learning-based framework for decentralized multi-agent exploration. ICMAPE jointly learns a centralized inference network $\Inf$, which maps global trajectory data to a posterior over hypotheses, and decentralized policies that collect data from local observation histories. Here, a hypothesis denotes a possible answer to the inference task, such as a target location or latent environment parameter. ICMAPE is inspired by In-Context Pure Exploration (ICPE)~\citep{russo2025}, which learns exploration policies across problem instances in the single-agent setting. Our work contributes to study what changes when the data-collection policy is restricted to the decentralized class
$\Pi_{\mathrm{dec}} = \left\{ \pi=(\pi^1,\ldots,\pi^N): \pi_t(a_t\mid \D_t) = \prod_{i=1}^{N} \pi_t^i(a_t^i\mid \D_t^i) \right\},$ 
while the hypothesis is still identified from the collective trajectory 
$\D_t=\{\D_t^i\}_{i=1}^{N}.$
This creates a coordination problem absent from the single-agent formulation: no agent observes the full evidence, yet the local policies must jointly produce informative and non-redundant data. Concretely, we (i) extend the fixed-confidence pure-exploration formulation from centralized to decentralized policy classes; (ii) show that the posterior over hypotheses admits a policy-independent likelihood kernel even when trajectories are generated by decentralized history-dependent policies (Lemma~\ref{lemma:posterior kernel}); (iii) incorporate stopping into the decentralized setting through augmented action spaces under a common stopping time (Lemma~\ref{lemma:stop_action} Proposition~\ref{prop:optimal_policy}); and (iv) give a CTDE learning procedure in which decentralized actors are optimized against the global inference objective (Section~\ref{sec:appendix_vdn} and ~\ref{sec:policy_learn}). 
MATD3 is one realization; other CTDE optimizers can be used with the same inference reward, dual-variable schedule, and stopping mechanism.

\section{Related Work}

\paragraph{Active Sequential Testing.}
Our work builds on active sequential hypothesis testing \citep{naghshvar2013Active} and fixed-confidence pure exploration \citep{kaufmann2016,garivier2016}, which adaptively collect data to identify a target hypothesis while minimizing expected stopping time subject to a correctness constraint. Classical approaches, including Chernoff-style testing, Track-and-Stop, and Top-Two sampling, derive sampling and stopping rules for specific statistical models \citep{naghshvar2013Active, kaufmann2016,garivier2016,russo2016Simple,vannella2024,degenne2019Nonasymptotic,jourdan2022top,wang2021Fast,russo2025Adaptive}. Recent work on In-Context Pure Exploration (ICPE) treats exploration as a learnable object across problem instances~\citep{russo2025pure}.

Closest to our setting, \citet{szostak2024deep} study decentralized active hypothesis testing with deep MARL, where agents select sampling and stopping actions and rewards derive from Bayesian risk. Their method requires posterior beliefs computed from known observation models, and neither the stopping action nor the learning objective is theoretically justified. We instead learn the inference rule, and derive both the stopping action and the objective from the fixed-confidence formulation (Proposition~\ref{prop:optimal inf}, Lemma~\ref{lemma:stop_action}, Proposition~\ref{prop:optimal_policy}).

\paragraph{Federated and Distributed Pure Exploration.}
In federated BAI, multiple agents explore bandit arms and communicate with a central server that aggregates information and identifies the arm with the highest mean reward across agents, while minimizing data exposure \citep{mitra2021, kota2023, Chen_2024}. Another related line studies multi-agent best arm identification with centralized action selection under structured reward models such as factor graphs \citep{vannella2023}. 
However, these formulations assume centralized action selection and target an optimal joint action. Our setting instead requires decentralized action selection from local histories, with agents jointly collecting data for broader inference tasks.

\paragraph{Cooperative Multi-Agent Reinforcement Learning.}
 We adopt the CTDE paradigm, where agents are trained in a centralized manner but execute using only local observations.
Under CTDE, value factorization methods (VDN \citep{sunehag2017}, QMIX \citep{rashid2020}, QTRAN \citep{son2019}) decompose a joint value function into per-agent components, while policy-gradient methods (MADDPG \citep{lowe2017}, COMA \citep{foerster2018}, MAAC \citep{iqbal2019}, MATD3 \citep{ackermann2019}) learn decentralized policies with centralized critics. Both are designed for reward maximization and do not address inference objectives.

\section{Background}

\paragraph{Environment Modeling.}
We consider a system with $N$ agents interacting with a common environment drawn from a model class $\mathcal{M}$, under a distribution $\mathcal{P} \in \Delta(\mathcal{M})$. Each environment is modeled as a tuple $M = (\mathcal{X}, \mathcal{A}, P, \rho)$, where $\mathcal{X} = \bigcup_{i=1}^N \mathcal{X}^i$ is the union of agent-specific observation spaces, and $\mathcal{A} = \bigcup_{i=1}^N \mathcal{A}^i$ is the union of agent-specific action sets. The transition function is 
$P = (P_t)_{t \in \mathbb{N}}$, where $P_t : (\mathcal{X} \times \mathcal{A})^t \to \Delta(\mathcal{X})$, and 
$\rho\in\Delta(\mathcal{X})$ is the initial observation distribution.

\paragraph{Agents and Decentralized Policies.}
Agents act simultaneously at each time step. Each agent $i \in \{1,\ldots,N\}$ receives a local observation $x_t^i \in \mathcal{X}^i$, selects an action $a_t^i \in \mathcal{A}^i$, and then receives the next local observation 
\(
x_{t+1}^i.
\)
Denote the joint action by $a_t=(a_t^1,\ldots,a_t^N)$. The local trajectory available to agent $i$ is
\(
\mathcal{D}_t^i = (x_1^i,a_1^i,\ldots,a_{t-1}^i,x_t^i).
\)
We denote the joint observation by
\(
x_t=(x_t^1,\ldots,x_t^N),
\)
and the full joint trajectory by
\(
\mathcal{D}_t=\{\mathcal{D}_t^i\}_{i=1}^N.
\)
We restrict attention to decentralized policies. Each agent follows a policy
$\pi^i=(\pi_t^i)_{t\in\mathbb{N}}$ satisfying
\(
\pi_t^i:\mathcal{D}_t^i \to \Delta(\mathcal{A}^i).
\)
Thus, agent $i$ samples actions according to
\(
a_t^i \sim \pi_t^i(\cdot \mid \mathcal{D}_t^i),
\)
using only its own local trajectory.

\paragraph{Hypothesis.} Within the collective observation space $\mathcal{X}$, we consider a predefined task-specific hypothesis class $\mathcal{H}$ and restrict our attention to the finite case (extensions to continuous classes are possible \cite{russo2026_c}). For any environment $M \in \mathcal{M}$, there exists a ground-truth hypothesis $H_M^\star \in \mathcal{H}$.  For example, in a source localization task, $H^\star_M$ may be the hidden location of the signal source. We drop the subscript $M$ and write $H^\star$ when the dependence is clear from context.

\paragraph{Objective.} Given a trajectory of observations collected by the agents, the goal is to accurately identify the ground-truth hypothesis. To formalize this, we introduce an inference function that maps observed data to a distribution over hypotheses. The inference function operates on the full joint trajectory $\mathcal{D}_t = \{\mathcal{D}_t^i\}_{i=1}^N$ and is defined as
\(
\Inf: \mathcal{D}_t \mapsto \Delta(\mathcal{H}).
\) 
Given a target error probability $\delta\in(0,1)$, the decentralized data-collection policies $\pi$ and the inference rule $\Inf$ induce a stopping time $\tau$, which denotes the number of joint interaction rounds before termination. The objective is to minimize the expected stopping time while ensuring that the final prediction is correct with probability at least $1-\delta$:
\begin{equation}
\label{eq:optimization_eq}
\min_{\pi, \tau, \Inf} \mathbb{E}^\pi [\tau]
\quad \textbf{s.t.} \quad 
\mathbb{P}^\pi\big(\Inf(\mathcal{D}_\tau) = H^\star\big) \geq 1 - \delta,
\end{equation}
in which $\mathbb{P}^\pi(\cdot)$ denotes the probability measure over trajectories induced by first sampling an environment $M \sim \mathcal{P}$ from the model class $\mathcal{M}$ and then executing policy $\pi$.

\section{Method}

Our method begins by reformulating the constrained identification problem through a Lagrangian relaxation. We then optimize over the inference rule, showing that the optimal rule selects the hypothesis with maximal posterior confidence. This yields a reward-based sequential decision problem with an auxiliary stopping action, which we solve using an actor–critic reinforcement learning framework under a CTDE paradigm.

\subsection{Lagrangian Relaxation and RL Reformulation}
\label{sec:rl_reformulation}

Introducing a multiplier $\lambda \geq 0$, the Lagrangian relaxation of Eq.~\ref{eq:optimization_eq} is given by

\begin{equation}
\label{eq:dual-relax1}
\inf_{\lambda \ge 0} \; \sup_{\pi, \mathcal{I}}
\left[
- \mathbb{E}_{\pi}[\tau]
+
\lambda \Big(
\mathbb{P}_{\pi}\big( \mathcal{I}(\mathcal{D}_\tau) = H^\star \big)
- 1 + \delta
\Big)
\right].
\end{equation}

For any fixed policy $\pi$, we can optimize over the inference rule. At any time $t \in \mathbb{N}$, given the observed trajectory $\mathcal{D}_t$, the optimal inference rule selects the hypothesis with the largest posterior probability (Proposition \ref{prop:optimal inf}):
\[
\Inf_t(\mathcal{D}_t)
\in
\arg\max_{H \in \mathcal{H}}
\mathbb{P}_t(H^\star = H \mid \mathcal{D}_t),\]
where $\mathbb{P}_t(\cdot \mid \mathcal{D}_t)$ denotes the posterior distribution over hypotheses. This induces the posterior confidence
\(
r_t(\mathcal{D}_t)
:=
\max_{H \in \mathcal{H}}
\mathbb{P}_t(H^\star = H \mid \mathcal{D}_t)
\), which represents the probability of correct identification under the optimal inference rule. Substituting it into the Eq.~\ref{eq:dual-relax1} yields
\(
\inf_{\lambda \ge 0}\ \sup_{\pi}
\mathbb{E}^\pi\!\left[
-\tau + \lambda\bigl(r_\tau(\mathcal{D}_\tau) - 1 + \delta\bigr)
\right].
\)

To express this objective as a sequential decision problem, we augment each agent's action space with an auxiliary stopping action $a_{\mathrm{stop}}$. At each time step, agents may either continue collecting data or trigger termination. We assume that stopping is shared across agents, meaning if any agent selects the stopping action, the data collection process terminates and the stopping signal is observed by all agents. This induces a common stopping time in Assumption~\ref{assumption:common_stopping} concerns the filtration rather than the agents' information sets, and it can be satisfied with no exchange of
observations or histories (Remark~\ref{rem:public_termination}).
\begin{assumption}[Common stopping time]
There exists a universal stopping time $\tau$ such that for every agent $i$, $\tau$ is a stopping time with respect to the local filtration $\mathcal{F}^i_t = \sigma(\mathcal{D}^i_t)$; i.e., for all $t \ge 1$, the event $\{\tau \leq t\} \in \mathcal{F}^i_t$. Equivalently,
\(
\{\tau \leq t\} \in \bigcap_{i=1}^N \mathcal{F}^i_t .
\)
\label{assumption:common_stopping}
\end{assumption}
Under this formulation, continuing exploration incurs a unit cost, while stopping receives a terminal reward based on the posterior confidence of the inferred hypothesis. We define the reward function as
\[
r_{t,\lambda}(\mathcal{D}_t, a_t)
=
-\mathbf{1}\{a_t \neq a_{\mathrm{stop}}\} +
\lambda\bigl(r_t(\mathcal{D}_t) - 1 + \delta\bigr)
\mathbf{1}\{a_t = a_{\mathrm{stop}}\}.
\]
The corresponding action-value function satisfies a Bellman recursion in which continuing incurs unit cost and stopping terminates with the confidence reward; the dual formulation and the resulting reward-maximization problem are derived in Appendix~\ref{pf:dual_formulation}.

\subsection{Inference Network}
\label{sec:inference_network}

We introduce an inference network that extracts task-relevant information from the global trajectory data collected by all agents.
Given the trajectory up to time $t$, denoted by $\mathcal{D}_t$, the inference network $\Inf_\psi$, parameterized by $\psi$, outputs a posterior distribution over hypotheses,
\(
p_\psi(H \mid \mathcal{D}_t).
\)
We obtain a point estimate via maximum a posteriori (MAP) inference, given by
\(
\hat{H}_t = \arg\max_{H\in\mathcal{H}} \; p_\psi(H \mid \mathcal{D}_t).
\)
The inference network is trained using the true hypothesis $H^\star$ sampled with each environment instance by minimizing the negative log-likelihood loss
\begin{equation}
\mathcal{L}_{\mathrm{inf}}
=
-\log p_\psi(H^\star \mid \mathcal{D}_t).
\label{eq:inf_loss}
\end{equation}

We define the policy-learning reward as the posterior confidence of the inferred hypothesis,
\(
r_t = p_\psi(\hat{H}_t \mid \mathcal{D}_t)
\).

\subsection{Critics}

We build on MATD3~\citep{ackermann2019}, which extends MADDPG~\citep{lowe2017} with the twin-critic architecture of TD3~\citep{fujimoto2018}. During training, the centralized critic conditions on the global trajectory and joint action, while the actor remains decentralized and depends only on the agent's local trajectory. The two Q-functions, $Q_{\phi_1}$ and $Q_{\phi_2}$, are used through a minimum operator in the Bellman target to reduce overestimation bias.  Since multiplication by the positive constant $\lambda$ does not change the optimal policy for fixed $\lambda$, we parameterize the tradeoff by a learnable per-step cost $\zeta \approx 1/\lambda$. Therefore, continuing exploration incurs cost of $-\zeta$, while stopping receives the terminal confidence reward.

For each agent $i$, the critic maps the global trajectory and the other
agents' actions to a vector of values over agent $i$'s action space, and
$Q_{\phi_j}^i(\mathcal{D}_t, a_t^i, a_t^{-i})$, $j \in \{1,2\}$, denotes
the coordinate selected by $a_t^i$.

We define the immediate cost as
\begin{equation}
c_t =
\begin{cases}
-\zeta, & \text{if no agent stops at time } t, \\[4pt]
r_t(\mathcal{D}_t) - 1+ \delta, & \text{otherwise,}
\end{cases}
\label{eq:cost}
\end{equation}
where $r_t(\mathcal{D}_t)$ is the posterior confidence of the inferred hypothesis. 
Let
\(
s_t=\mathbf{1}\{\exists i,\ a_t^i=a_{\mathrm{stop}}\}
\)
denote whether any agent stops at time $t$. Then the TD3 target for agent $i$ is
\[
y_t^i =
 c_t
+
\gamma(1-s_t)
\min_{\ell \in \{1,2\}}
Q_{\phi_\ell^{\mathrm{tar}}}^i
\bigl(
\mathcal{D}_{t+1},
\tilde  a_{t+1}^{i,\text{tgt}},
\tilde  a_{t+1}^{-i,\text{tgt}}
\bigr),
\]
where $\gamma$ is the discount factor, $Q_{\phi_1^{\mathrm{tar}}}$ and $Q_{\phi_2^{\mathrm{tar}}}$ are target critics, and $\tilde a_{t+1}^{\text{tgt}}$ is obtained from the target actor.
Given the batch-level posterior confidence $\widehat r_t$, we update $\zeta$ using the dual-variable rule $e_\zeta = \widehat r_t - (1-\delta)$ and $\zeta \leftarrow \max\{0,\zeta+\eta e_\zeta\}$, where $\eta$ is the learning rate. We initialize $\zeta=0$, so the exploration cost is initially inactive and increases only when the posterior confidence exceeds the target level $1-\delta$.

Finally, the critic parameters are updated by minimizing the temporal-difference loss
\begin{equation}
\mathcal{L}_{\mathrm{critic}}
=
\mathbb{E}\!\left[
\sum_{i=1}^N \sum_{\ell=1}^2
\bigl(
Q_{\phi_\ell}^i(\mathcal{D}_t,a_t^i,a_t^{-i}) - y_t^i
\bigr)^2
\right].
\label{eq:critic_loss}
\end{equation}

\subsection{Actors}

For each agent $i$, the actor parameterizes a decentralized policy on the local trajectory $\mathcal{D}_t^i$. 
The actor maps a learned representation of $\mathcal{D}_t^i$ to unnormalized action scores $g_t^i=(g_t^i(a))_{a\in\mathcal{A}^i}$.
At execution time, we use the deterministic policy
\(
a_t^i
=
\bm{\mu}_{\theta_i}(\mathcal{D}_t^i)
=
\arg\max_{a\in\mathcal{A}^i} g_t^i(a).
\)
During actor training, however, we use a Gumbel--Softmax relaxation to obtain differentiable action representations $\widetilde a_t^i$, allowing gradients to pass through the discrete action-selection process.

The decentralized actors are trained to maximize the centralized critic value under the relaxed joint action $\widetilde a_t=(\widetilde a_t^1,\ldots,\widetilde a_t^N)$. We define the actor objective
\begin{equation}
J(\theta)
=
\mathbb{E}\!\left[
\sum_{i=1}^N
Q_{\phi_1}^i
\bigl(
\mathcal{D}_t,
\widetilde a_t^i,
\widetilde a_t^{-i}
\bigr)
\right],
\label{eq:actor_obj}
\end{equation}
and optimize the actors by minimizing $\mathcal{L}_{\mathrm{actor}}(\theta)=-J(\theta)$.

\paragraph{Information structure.} During training, each actor
conditions only on its own local trajectory $\mathcal{D}_t^i$, the critic
on the global trajectory and joint action, and the inference network on
the global trajectory, supervised by the true hypothesis $H^\star$ of the
sampled environment. At execution the critic is discarded and $H^\star$ is
not required: agents act from local trajectories, and the inference
network aggregates the collected data once, after termination.
Pseudocode is in Algorithms~\ref{alg:multicpe_td3_training} and~\ref{alg:multicpe_inference};
architectures in Appendix~\ref{sec:architecture_appendix}.

\section{Experiments}
\label{sec:exp}

We evaluate ICMAPE-TD3 in two synthetic benchmark environments designed to isolate key challenges in decentralized information acquisition. First, we study bandit capacity identification, where agents must infer how many agents each arm can support before becoming overloaded. This task requires coordinated joint actions to reveal meaningful information about the hypothesis and is motivated by shared resource systems such as communication channels with unknown user capacity. We consider two bandit settings: a 3-arm-5-player setting (A3P5) and a 5-arm-3-player setting (A5P3).  
Second, we study signal source localization with distractors, motivated by distributed sensor localization  applications such as environmental monitoring, bush fire surveillance, and precision agriculture, where agents coordinate exploration to identify the true signal source~\cite{MAO2007,patwari2005locating}. Third, we evaluate on a Maryland nitrate monitoring task built from real measurements, which tests transfer to real observation noise. In all experiments, we set the confidence level to $\delta = 0.1$, corresponding to a target accuracy of $1-\delta = 0.9$.
\subsection{Multi-armed Bandits Capacity Identification}
\label{subsec:bandit_id}

We consider a setting with \(K\) arms and \(N\) players. Inspired by multiplayer multi-armed bandit models with non-unit arm capacities~\citep{wang2022, wang2022a,hu2026}, we study the problem of identifying the capacity of all arms \(k\in[K]\). 

Each arm $k$ has an unknown capacity $C_k^\star \in \{1, \dots, N\}$, which is sampled independently from a uniform distribution over $\{1, \dots, N\}$, defined as the maximum number of agents that can simultaneously select arm $k$ without violating its constraint.   At each time step $t$, each agent $i$ chooses either an arm to pull or the stopping action, so that
\(
a_t^i \in [K] \cup \{a_{\mathrm{stop}}\}.
\)
The objective is to identify the true capacity vector $C^\star = (C_1^\star, \dots, C_K^\star)$  with confidence level $1 - \delta$. The hypothesis space is defined as $\mathcal{H} = \{1, \dots, N\}^K$, where each hypothesis corresponds to a possible capacity configuration over the $K$ arms.

Let $n_t^k = \sum_{i=1}^N \mathbb{I}\{a_t^i = k\}$ denote the number of agents selecting arm $k$ at time $t$. When $n_t^k > C_k$, the arm is said to be overloaded, and the resulting observations are zero environment rewards on that arm. Otherwise, each agent $i$ pulling arm $k$ receives an outcome
\(
x_t^i \sim \mathrm{Bernoulli}(\mu_k),
\)
where $\mu_k \in (0,1]$. To ensure that the problem remains tractable and avoid an impractically large number of samples would be required to reliably estimate the capacity, we assume the mean reward is sampled as
\(
\mu_k \sim \mathrm{Unif}[0.3,1]
\) for all arms.

\subsection{Signal Source Localization with Distractors}
\label{subsec:signal_localization}
In the second environment, we model a cooperative signal-localization problem in which heterogeneous agents seek to identify an unknown signal source through decentralized exploration. We consider a discrete grid world of size $n\times n$, with $N=3$ agents. The objective is to identify the true signal source location $s^\star\in[n]\times[n]$ with confidence level $1-\delta$, where $\delta\in[0,1)$.

The environment contains one true signal source \(s^\star=(r^\star,c^\star)\in[n]\times[n]\) and one distractor source \(s'=(r',c')\in[n]\times[n]\). The hypothesis space is \(\mathcal{H}=[n]\times[n]\), where each hypothesis corresponds to a possible true source location. Both sources generate higher signal intensity than nearby cells. The true source has peak intensity \(1\), while the distractor has peak intensity \(\xi\), where \(\xi \sim \mathrm{Unif}[0,0.7]\). To ensure that the two sources are spatially distinguishable, we impose \(d_{\mathrm{Man}}(s^\star,s') \geq r\).

At each time step \(t\), agent \(i\) occupies a location \(p_t^i \in [n]\times[n]\) and chooses either a movement action or the stopping action. A movement action displaces the agent by \(\Delta_{\mathrm{move}}\). 
Each agent observes only through its own sensing modality, with no access to other agents' measurements.
Agent \(i\) observes an \(m\times m\) local window centered at \(p_t^i\), clipped at the grid boundary. For Agents 1 and 2, the observation at a queried cell \(u\) in this window follows
\[
Y^{(i)}(u) = B^{(i)}(u) \cdot \mathbb{I}\{u = s^\star\},
\qquad B^{(i)}(u) \sim \mathrm{Bernoulli}(q_i),
\]
where \(s^\star\) is the target cell and \(q_i\) is the probability that agent \(i\) registers a detection when it queries the target cell. Agent  follows a different observation model, described below.

Agent~3 observes a different and noisy version of the latent signal field. For each cell $u\in[n]\times[n]$, let $d_{\min}(u)=\min\{d_{\mathrm{Man}}(u,s^\star),d_{\mathrm{Man}}(u,s')\}$ and define $s(u)=I(u)\exp(-\lambda d_{\min}(u)),$ where $I(u)=1$ for cells closer to $s^\star$ and $I(u)=\xi$ otherwise, and decay rate $\lambda>0$. Agent~3 observes
\[
Y^{(3)}(u)=\mathrm{clip}_{[0,1]}\bigl(s(u)+\epsilon(u)\bigr),
\text{ where }
\epsilon(u)\sim\mathcal{N}(0,\sigma^2).
\] Examples of noisy observations and environment parameters are in Appendix  Figure~\ref{fig:signal_map_noise} and Table~\ref{tab:signal_env_params}.

\subsection{Maryland Nitrate Concentration Monitoring}
\label{subsec:nitrate}

We evaluate ICMAPE-TD3 on a real-world environmental monitoring task using Maryland nitrate concentration data spanning 2001--2023, with training on 2001--2018 and evaluation on 2019--2023. We obtained Maryland nitrate concentration measurements from the Water Quality Portal (WQP)~\citep{waterqualityportal2026marylandnitrate}, a public database integrating water-quality observations from USGS, EPA, and state/local monitoring agencies. 
Elevated nitrate levels contaminate drinking water sources and pose direct health risks, particularly for infants~\citep{epa2026nitrate_groundwater}, while also contributing to nutrient pollution, algal blooms, and oxygen depletion that threaten aquatic ecosystems~\citep{usgs_nutrients_eutrophication}. Therefore, it is  important to develop methods that can quickly identify high-concentration regions and support timely intervention. 

For each year, a subset of real monitoring measurements is first used to fit a Gaussian process, producing a spatial concentration field $f: \mathcal{X} \to \mathbb{R}$ over a discretized map $\mathcal{X}$ of the region. The true hypothesis $\theta^\star \in \mathcal{H} = \mathcal{X}$ is defined as the location of the highest concentration, i.e., $\theta^\star = \arg\max_{x \in \mathcal{X}} f(x)$. 
Full experimental details and results are provided in Appendix~\ref{app:nitrate}.

\subsection{Baselines}
We compare against three baselines used across all experiments: random exploration, Independent Q-learning, and VDN. These baselines span different coordination assumptions and contrast with our policy-gradient actor--critic approach. In addition, we include experiment-specific baselines: environment-reward-guided exploration and centralized coordination for the bandit task, Thompson Sampling for signal-source localization, and GP Prior Map Guided Sampling for the nitrate monitoring task. All baselines are described in Appendix~\ref{baseline:gp_prior}.

\section{Results}

\subsection{ Multi-armed Bandit Capacity Identification}

\paragraph{Training.}  In the A3P5 setting, both ICMAPE-TD3 and ICAPE-VDN are able to reach the target accuracy early in training. However, once the reward structure begins to favor shorter trajectories more strongly, ICMAPE-TD3 adapts more effectively by reducing the average stopping time.  This suggests that ICMAPE-TD3 learns an effective sampling strategy that reduces stopping time after reaching the target accuracy. 
In the A5P3 setting, the capacity identification problem becomes more challenging due to the larger hypothesis space, and only ICMAPE-TD3 is able to reliably reach the target accuracy and improve the sampling strategy (Figure~\ref{fig:training_dynamics_bandit}).
\begin{figure}[ht]
    \centering
        \includegraphics[width=\linewidth]{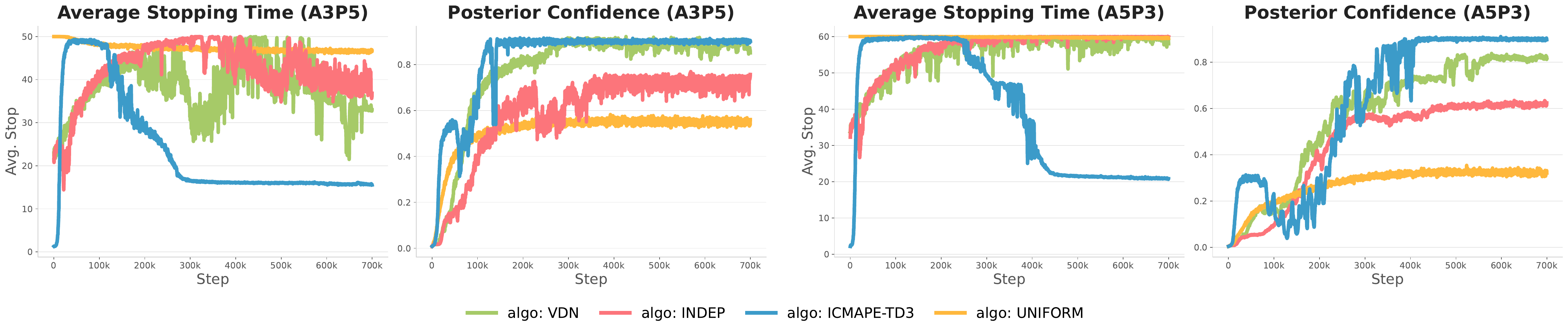}
    \caption{Training dynamics in A3P5 and A5P3.
        Curves are averaged over three random seeds and smoothed with an exponential moving average parameter of $0.2$.
    }
    \label{fig:training_dynamics_bandit}
\end{figure}

\paragraph{Evaluation.} We evaluate the final performance of each method in terms of estimation accuracy and average stopping time across 10,000 sampled environments over three seeds. The results are summarized in Figure~\ref{fig:final_performance_bandit}. ICMAPE-TD3 achieves stable accuracy around $1-\delta$ while maintaining significantly lower stopping times compared to baselines.

\begin{figure}[t]
    \centering
    \includegraphics[width=\linewidth]{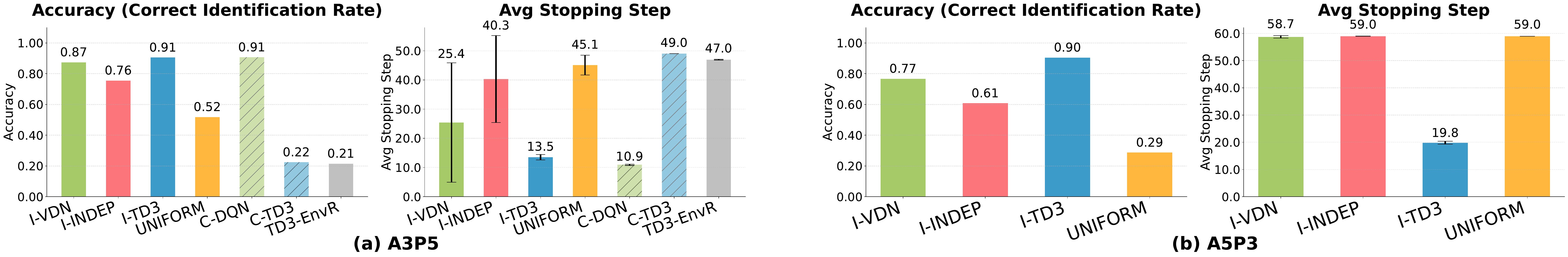}
    \caption{
   Bandit  Capacity Accuracy and Average Stopping Step. ``I'' denotes ICMAPE. ``C'' denotes centralized baselines.
    In (a) A3P5, the TD3-EnvR bar represents TD3 trained to maximize the
    environmental reward Error bars indicate one standard deviation.
    }
    \label{fig:final_performance_bandit}
\end{figure}

\paragraph{Coordination Strategy.}
The learned policy follows a highly structured probing pattern. Early joint actions are identical across active environments, first concentrating players on one arm and then reallocating subsets across arms (Appendix Table~\ref{tab:joint_action_freq}). This suggests that agents learn to test capacity constraints through coordinated joint allocations rather than independent random exploration. Once sufficient evidence is collected, one or more agents select the stop action (Appendix Table~\ref{tab:stop_combinations}).

\paragraph{Robustness Analysis.}
We evaluate the robustness of ICMAPE-TD3 under different priors over arm means and arm capacities. For the arm-mean prior, we replace the uniform prior with a scaled Beta distribution parameterized by $(\alpha,\beta)$, where $\alpha,\beta\in\{0.5,1,3,5,7\}$. When $\alpha>\beta$, the prior places more mass on larger arm means, yielding more informative feedback and average accuracy above $0.95$; when $\alpha\leq\beta$, the accuracy ranges from $0.67$ to $0.96$.
For the arm-capacity prior, we sample capacities from a distribution induced by $p\sim\mathrm{Dir}(\theta\mathbf{1}_K)$, with $\theta\in\{0.1,0.5,1,5,10\}$. Smaller $\theta$ values produce more sparse capacity distributions, while larger values approach a more balanced distribution. ICMAPE-TD3 remains robust across these capacity priors. Figures~\ref{fig:beta_heatmap} and~\ref{fig:dirichlet_robustness} in the appendix report the average inference accuracy under different arm-mean and arm-capacity priors.

\paragraph{Ablation Studies.}

Ablations in Table~\ref{tab:ablation} isolate the contributions of the inference-driven reward,
the learned stopping mechanism, and the per-step sampling cost. Together,
they support the paper's main claim that ICMAPE achieves the target identification accuracy with fewer samples by jointly learning \emph{what}
information to collect and \emph{when} continued exploration is no longer
necessary.

\begin{table}[t]
  \caption{Ablation study on ICMAPE-TD3 in A3P5.}
  \label{tab:ablation}
  \centering
  \small
  \begin{tabular}{lcccrr}
    \toprule
    Variant & Inf.\ Reward & Learned Stop & Step Cost & Accuracy & Avg.\ Stop Time \\
    \midrule
    Full ICMAPE-TD3               & \yes & \yes & \yes & \textbf{0.91} & \textbf{13.5} \\
    Environment-reward TD3 & \no  & \yes & \yes & 0.21 & 47.0 \\
    Fixed-threshold stopping      & \yes & \no  & \yes & 0.70 & 37.3 \\
    No sampling cost              & \yes & \yes & \no  & 0.90 & 50.0 \\
    \bottomrule
  \end{tabular}
\end{table}
Replacing the inference-driven reward with the environment reward
substantially reduces identification accuracy. This result shows that
maximizing the environment reward does not necessarily produce trajectories
that are informative for identifying the underlying environment hypothesis.
After replacing learned stopping with a fixed threshold also  degrades performance and produces a policy that fails to reach the target accuracy, this indicates stopping should be optimized jointly with exploration. The result is consistent with our theoretical formulation, in which the optimal policy compares the $Q$-value of stopping with those of continued exploration.
Finally, without the per-step sampling cost, the policy has no incentive to stop and learns unnecessary exploration.

\paragraph{Posterior Calibration}
To assess how well the inference module's output probabilities align with
empirical accuracy, we compute the Expected Calibration Error \citep{naeini2015obtaining} on
A3P5 over $10{,}000$ trajectories, pooling the posterior predictions from all time steps, obtaining $\mathrm{ECE} = 0.0065$. The gap between mean
predicted confidence and empirical accuracy is at most $0.02$ in every bin,
see  Appendix~\ref{sec:posterior_cal} Table~\ref{tab:calibration}.

\subsection{Signal Source Localization with Distractors}

\paragraph{Evaluation.}
 Training dynamics of this settings is in Appendix Figure~\ref{fig:training_dynamics_signal}. ICMAPE-TD3 is the only method that reliably reaches the target accuracy while further improving efficiency by reducing the average stopping time. We evaluate the performance of each method across 10,000 sampled environments. The results are summarized in Figure~\ref{fig:final_performance_signal}.  ICMAPE-TD3 achieves accuracy close to the target confidence level $1-\delta$ while uses fewer samples than the baselines.

\paragraph{Coordination Strategy.} In this task, the learned policy exhibits an implicit division of roles. Agents 1 and 2 follow deterministic trajectories across environments, while Agent 3 adapts its actions based on its local observations of the noisy signal field (Appendix Table~\ref{tab:robot_per_player_action_freq}). Stopping is also role-specific, with more than one agent stopping simultaneously in fewer than $10\%$ of sampled environments (Appendix Table~\ref{tab:robot_stop_combinations}).

\begin{figure}[t]
\centering
\includegraphics[width=0.6\linewidth]{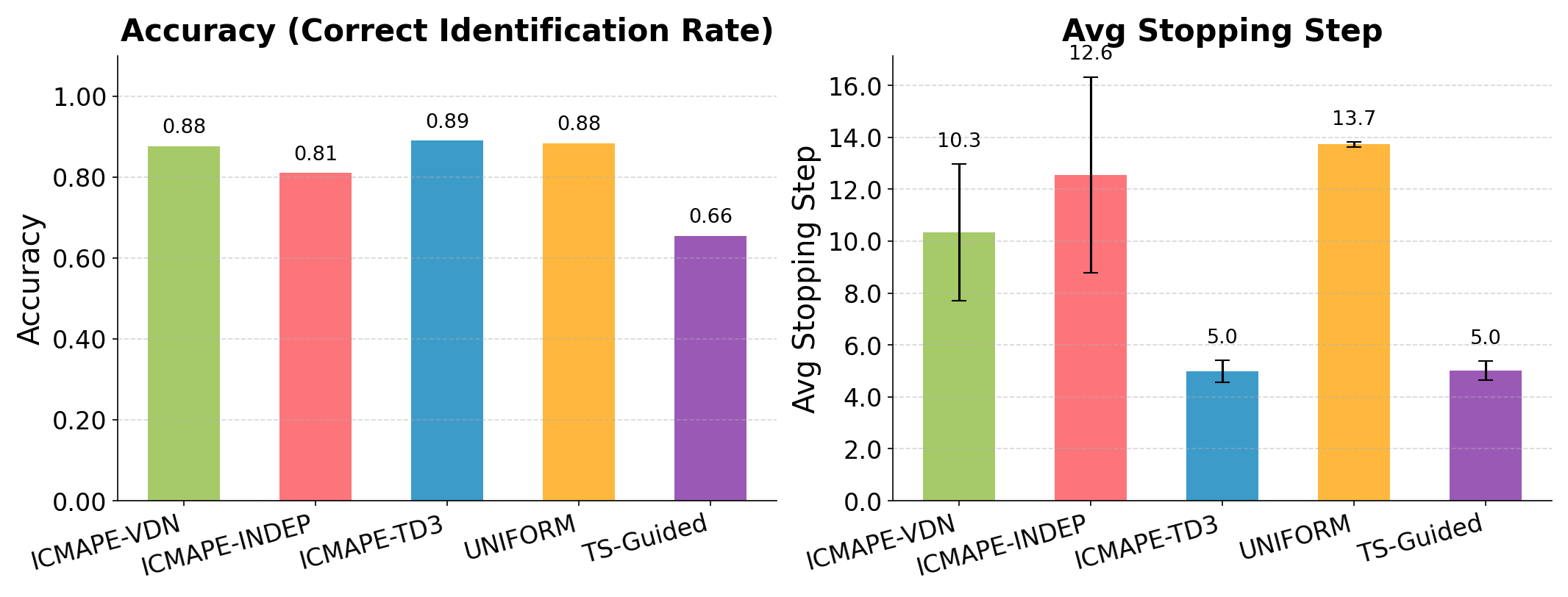}
\caption{Signal Source Localization Accuracy and Average Stopping Step. The TS bar is the Thompson Sampling guiding exploration strategy. Error bars indicate one standard deviation.}
\label{fig:final_performance_signal}
\end{figure}

\paragraph{Robustness Analysis.}
We evaluate the robustness of ICMAPE-TD3 under different priors over distractor intensity and different distances between the distractor and target location; see Appendix Figure~\ref{fig:dist_acc} and ~\ref{fig:signal_intensity_beta}. For distractor intensity, we replace the uniform prior with a scaled Beta distribution parameterized by $(\alpha,\beta)$, using the same set of parameter values as above. When $\alpha>\beta$, the prior places more mass on higher distractor intensities, leading to stronger distractions while still achieving an average accuracy between $0.8$ and $0.9$; when $\alpha\leq\beta$, the accuracy remains around $0.9$.
Accuracy also improves as the Manhattan distance between the distractor and the true source grows, staying above 
$0.85$ throughout

\subsection{Nitrate Concentration }
We evaluate the proposed method on the Maryland nitrate concentration environment, where annual nitrate fields are generated using Gaussian Process (GP) models fitted to historical nitrate observations. For each year, we sample 100 GP-generated nitrate maps with 1000 spatial points per map, and report the average performance for identifying the highest-concentration region.

Figure~\ref{fig:nitrate_results} reports the correct identification accuracy and the average number of samples. ICMAPE-TD3 achieves accuracy close to the target level on both training and test years, with only a small performance gap between them. This suggests that the learned policy is robust to GP-generated nitrate fields from unseen years. In contrast, the independent, VDN, and uniform baselines show larger drops in test-year performance, indicating weaker generalization. The GP-prior baseline achieves the highest accuracy by directly using spatial prior information.

In terms of sample efficiency, ICMAPE-TD3 requires the fewest samples among all compared methods on both training and test years. Although the GP-prior baseline reaches slightly higher accuracy, it uses more samples than ICMAPE-TD3. Overall, ICMAPE-TD3 provides the best trade-off between accuracy and sampling cost, achieving near-target identification performance while maintaining the lowest sample complexity and stable generalization to held-out years.

\begin{figure}[t]
    \centering
    \includegraphics[width=0.6\linewidth]{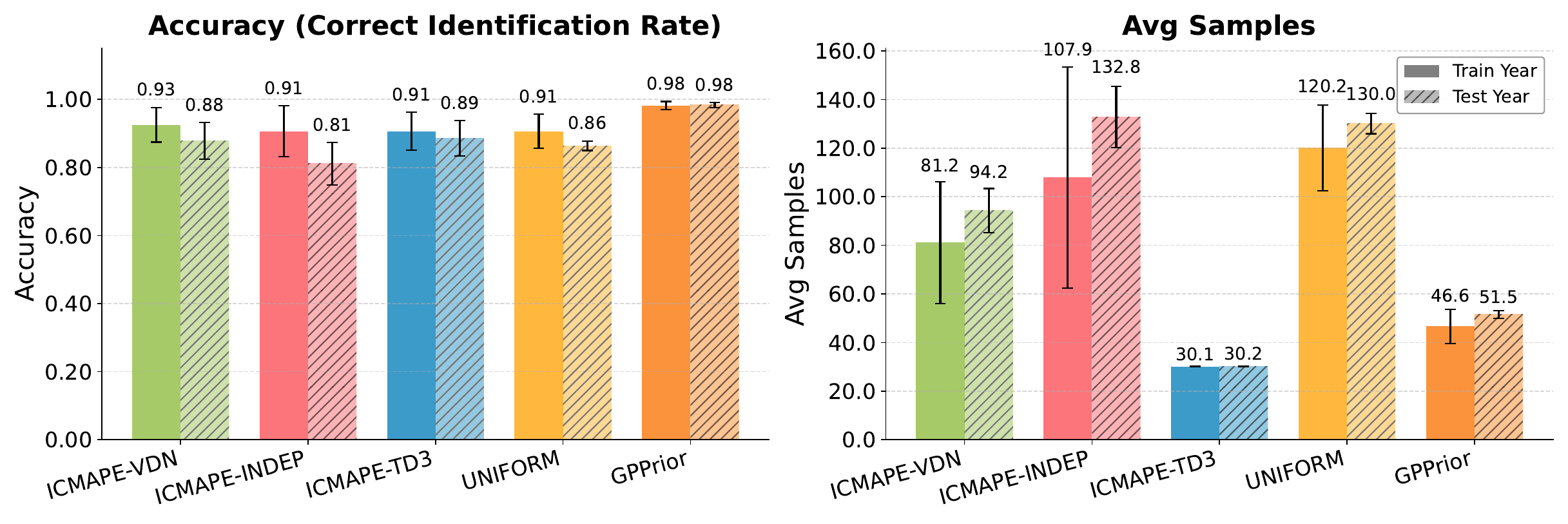}
    \caption{
   Nitrate Concentration Accuracy and Average Samples. Error bars indicate one standard deviation.
    }
    \label{fig:nitrate_results}
\end{figure}

\vspace{-0.01em}
\section{Conclusion}
In this work, we study a decentralized multi-agent pure exploration problem, where the objective is not reward maximization, but efficient information acquisition for identifying unknown properties of the environment. While existing pure exploration methods provide strong foundations, they are often tailored to specific statistical models or assume centralized action selection, leaving a gap in general learning-based methods for decentralized multi-agent data collection. We propose ICMAPE to fill this gap, a method that combines a centralized inference network with decentralized data collection policies. The inference network estimates a posterior distribution over hypotheses from global trajectory data, while each agent selects actions using only its local observation history. We evaluate ICMAPE on two synthetic benchmarks---bandit capacity identification and signal source localization---and on a real-world Maryland nitrate concentration monitoring task, demonstrating its ability to transfer to real-world observation noise and non-stationary dynamics.
We reformulate the fixed-confidence optimization problem as a learnable reinforcement learning objective and train the resulting model under the CTDE paradigm.   Future research could extend this setting toward more general forms of environment understanding. We discuss limitations and broader impacts in Section~\ref{sec:limitations_broader_impact}.

\clearpage
\bibliographystyle{unsrtnat} 
\bibliography{ref}

\newpage
\tableofcontents
\newpage
\appendix

\section{Theoretical Details}
\label{apx:theory_detail}

Our theoretical formulation follows from the In-Context Pure Exploration~\citep{russo2025}.

\subsection{Preliminaries}

\begin{assumption}[Common dominating measures]
There exist reference measures \(\lambda_0\) and \(\lambda\) on \((\mathcal{X},\mathcal{B(\mathcal{X}))}\) such that, for all $(\rho,P)\in \mathcal{M}, s\in \mathbb{N}$, and $(z,a)\in\mathcal{Z}\times \A $,
\[
\rho(\cdot)\ll\lambda_0(\cdot) \quad and \quad P_s(\cdot \mid z,a) \ll\lambda(\cdot). 
\]
That is, all initial distributions and transition kernels have densities with respect to common reference measures:
\[
\rho^\omega(dx)=p^\omega_0(x)\lambda_0(dx),
\qquad
P^\omega_s(dx\mid z,a)=p^\omega_s(x\mid z,a)\lambda(dx).
\]
\label{assumption: domination measure}
\end{assumption}
\begin{lemma}[Posterior kernel under decentralized policies]
\label{lemma:posterior kernel}
Fix a time \(t\in\mathbb{N}\). Consider a model class \(\mathcal{M}\) with prior \(\mathcal{P}\), and let \(M\sim\mathcal{P}\) denote the unknown environment. Under a decentralized joint policy
\[
\pi = (\pi^1,\ldots,\pi^N),
\qquad
\pi_s^i(\cdot \mid \mathcal{D}_s^i)\in \Delta(\mathcal{A}^i),
\]
each agent \(i\) selects action \(a_s^i\) using only its local trajectory \(\mathcal{D}_s^i\). Let the global trajectory up to time \(t\) be
\[
\mathcal{D}_t = (x_1,a_1,x_2,a_2,\ldots,x_t),
\]
where \(x_s=(x_s^1,\ldots,x_s^N)\) and \(a_s=(a_s^1,\ldots,a_s^N)\).

Then there exists a probability kernel
\[
R_t:\mathcal{D}_t \times \mathcal{B}(\mathcal{M}) \to [0,1],
\]
independent of the decentralized policy \(\pi\), such that for every measurable set \(A\subseteq\mathcal{M}\) and every measurable trajectory event \(Z\),
\[
\mathbb{P}^{\pi}(M\in A,\mathcal{D}_t\in Z)
=
\int_Z R_t(A\mid z)\,
\mathbb{P}^{\pi}(\mathcal{D}_t\in dz).
\]
Consequently, for any hypothesis set \(B\subseteq\mathcal{H}\),
\[
\mathbb{P}^{\pi}(H^\star\in B\mid \mathcal{D}_t=z)
=
R_t\bigl(\{M\in\mathcal{M}:h^\star(M)\in B\}\mid z\bigr),
\]
for \(\mathbb{P}^{\pi}\)-almost every trajectory \(z\).
\end{lemma}
\begin{proof}

Let \(z=(x_1,a_1,x_2,\ldots,a_{t-1},x_t)\in\mathcal{Z}_t\) denote a realized global trajectory. Under a decentralized policy \(\pi\), the trajectory law induced by model \(M\) factorizes as
\[
\mathbb{P}^{\pi}_{M,t}(dz)
=
\rho_M(dx_1)
\prod_{s=1}^{t-1}
\left[
\prod_{i=1}^N
\pi_s^i(da_s^i\mid \mathcal{D}_s^i)
\right]
P_{M,s}(\mathrm{d}x_{s+1}\mid \mathcal{D}_s,a_s).
\]
The policy-dependent factor
\[
\prod_{s=1}^{t-1}\prod_{i=1}^N
\pi_s^i(a_s^i\mid \mathcal{D}_s^i)
\]
depends on the realized trajectory \(z\), but not on the model \(M\). Hence, once we condition on the trajectory, the action-selection probabilities do not provide additional likelihood information about \(M\)

We choose reference measures \(\lambda_0\) and \(\lambda\) such that
\[
\rho_M(dx_1)=r_M(x_1)\lambda_0(dx_1),
\qquad
P_{M,s}(dx_{s+1}\mid \mathcal{D}_s,a_s)
=
p_{M,s}(x_{s+1}\mid \mathcal{D}_s,a_s)\lambda(dx_{s+1}).
\]
Define a reference measure on the trajectory space by
\[
\nu_t^\pi(dz)
:=
\lambda_0(dx_1)
\prod_{s=1}^{t-1}
\left[
\prod_{i=1}^N
\pi_s^i(da_s^i\mid \mathcal{D}_s^i)
\right]
\lambda(dx_{s+1}).
\]
Then \(\mathbb{P}^{\pi}_{M,t}\ll \nu_t^\pi\) by Assumption~\ref{assumption: domination measure}, and the Radon--Nikodym derivative is
\[
\ell_t(M,z)
:=
\frac{d\mathbb{P}^{\pi}_{M,t}}{d\nu_t^\pi}(z)
=
r_M(x_1)
\prod_{s=1}^{t-1}
p_{M,s}(x_{s+1}\mid \mathcal{D}_s,a_s).
\]
The policy terms are absorbed into the reference measure \(\nu_t^\pi\), so the likelihood \(\ell_t(M,z)\) depends only on the model-dependent initial density, transition densities, and observed trajectory, but not on the decentralized policy.

Therefore, for any measurable \(A\subseteq\mathcal{M}\) and \(Z\subseteq\mathcal{Z}_t\),
\[
\mathbb{P}^{\pi}(M\in A,\mathcal{D}_t\in Z)
=
\int_A\int_Z
\ell_t(M,z)\,\nu_t^\pi(dz)\,\mathcal{P}(dM),
\]
and
\[
\mathbb{P}^{\pi}(\mathcal{D}_t\in Z)
=
\int_Z
\int_{\mathcal{M}}
\ell_t(M,z)\,\mathcal{P}(dM)\,\nu_t^\pi(dz).
\]
Thus the posterior kernel is
\[
R_t(A\mid z)
:=
\frac{
\int_A \ell_t(M,z)\,\mathcal{P}(dM)
}{
\int_{\mathcal{M}} \ell_t(M,z)\,\mathcal{P}(dM)
}.
\]
This kernel is independent of \(\pi\). Equivalently,
\[
\mathbb{P}^{\pi}(M\in A\mid \mathcal{D}_t=z)
=
R_t(A\mid z),
\]
for \(\mathbb{P}^{\pi}\)-almost every \(z\). The desired conclusion follows by taking
\[
A=\{M\in\mathcal{M}:h^\star(M)\in B\}.
\]
\
\end{proof}

\begin{proposition}
    Consider a fixed decentralized policy $\pi \in \prod_{\text{dec}}$. Let $t\in \mathbb{N}$, and $z\sim \prob_t^\pi$. For any $t$, the optimal inference rule of $\sup_{\Inf_t}\prob_t^\pi(H^\star = \Inf_t(\D_t))$ is given by $\Inf_t^*(z)=\argmax_{H\in \mathcal{H}}\prob^\pi_t(H^\star=H\mid \D_t=z)$
    \label{prop:optimal inf}
\end{proposition}
The proof of Proposition~\ref{prop:optimal inf} follows from Proposition~B.2 of \citet{russo2025}.

\subsection{Optimization Problem}
\paragraph{Fixed Confidence:} Given a target error $\delta\in(0,1)$, the learner chooses: 
a stopping time $\tau\in\mathbb{N}$, which is the total number of queries; a data collection 
policy $\pi$; and an inference rule $\Inf$ that minimize the expected total number of queries 
$\tau$ (see Assumption~\ref{assumption:common_stopping}), while meeting the correctness 
guarantee. $\tau$ marks the random moment when the data collection process stops.
We define the following optimization problem:
\[\min_{\cpi,\tau,\Inf}\E^\pi [\tau]\quad \textbf{s.t.} \quad 
\prob^{\cpi}(\Inf_\tau(\D_\tau)=H^\star)\geq 1-\delta,\]
where \(\mathbb{P}^\pi(\cdot)\) is the probability of the underlying data collection process 
when $\pi$ gathers data from environment $M$, and $M$ is sampled from the prior $\prior$. 
In the appendix, we use the notation $\cpi$ to represent policies in the decentralized 
policy class $\Pi_\text{dec}$.

\paragraph{Decentralized policy class.}
Since agents operate under decentralized observation, we restrict to decentralized policies 
$\Pi_{\mathrm{dec}}$ such that for every agent $i$ and time $t$, $a_i(t)$ depends only on 
its own local observation history $\mathcal{D}_{t-1}^i$ and is conditionally independent of 
all other agents' histories given its own:
\[
    a_i(t) \sim \pi_i(\cdot \mid \mathcal{D}_{t-1}^i).
\]
Furthermore, $\pi_i$ is measurable with respect to $\sigma(\mathcal{D}_{t-1}^i)$.

\subsection{Dual Problem Formulation and RL Conversion}
\label{pf:dual_formulation}

\paragraph{Dual Problem Formulation}
In the fixed confidence setting, we are solving the following optimization problem. 
\begin{equation}
\begin{aligned}
\inf_{\pi\in\Pi_{\mathrm{dec}},\,\tau,\,I}\quad & \mathbb{E}^{\pi}[\tau] \\
\text{s.t.}\quad 
& \mathbb{P}^{\pi}\!\left(I_\tau(\mathcal D_\tau)=H^\star\right)\ge 1-\delta,\\
& \mathbb{E}^{\pi}[\tau] < \infty,
\end{aligned}
\label{eq:fixed_optim_appendix}
\end{equation}
where we assumed \(\tau\) to be locally adapted to \((\sigma(\D_t^i))_t\) according to Assumption~\ref{assumption:common_stopping}. Here  \(\D_t^i = \{(i,a_s^i,x_s^i)\}_{s=1}^t\), denotes the local state of player 
$i$, i.e., the sequence of tuples containing the player index together with its local observation–action trajectory up to time 
$t$.

\begin{lemma}
    The Lagrangian dual of the problem in Eq.~\ref{eq:fixed_optim_appendix} is given by 
    \[
    \sup_{\lambda\geq0} \inf_{\cpi,\tau,\Inf} \mathcal{L}(\cpi,\tau,\Inf;\lambda) = \sup_{\lambda\geq0} \inf_{\pi,\tau,\Inf} \lambda(1-\delta) +\mathbb{E}^{\cpi} \Bigl[\tau-\lambda\prob_{\cpi}(H^\star=\Inf_\tau(\D_\tau)\mid \D_\tau\Bigr],
    \]
    where  $\lambda$ is the Lagrange multiplier.
\end{lemma}

\begin{proof}
    Since any feasible solution stops almost surely under assumption, we use $\cpi$ to mean decentralized policies and get
\begin{align}
\prob^{\cpi}\!\left(H^\star = \Inf_\tau(\D_\tau)\right)
&= \sum_{t \ge 1} \prob^{\cpi}_t\!\left(\Inf_t(\D_t)=H^\star,\ \tau=t\right)
\qquad \text{(law of total probability)}\\
&= \sum_{t \ge 1} 
\E^{\cpi}_t\!\left[
\mathbf{1}_{\{\tau=t\}}
\mathbf{1}_{\{H^\star=\Inf_t(\D_t)\}}
\right] \\
&= \sum_{t \ge 1} 
\E^{\cpi}_t\!\left[
\E_t[\mathbf{1}_{\{\tau=t\}}
\mathbf{1}_{\{H^\star=\Inf_t(\D_t)\}} \bigm|\D_t
] \right], \qquad  \text{(tower rule)}\\
&= \sum_{t \ge 1} 
\E^{\cpi}_t\!\left[
\mathbf{1}_{\{\tau=t\}}
\E^{\cpi}_t\!\left[
\mathbf{1}_{\{H^\star=\Inf_t(\D_t)\}}
\mid \D_t
\right]
\right]
\qquad (\{\tau=t\}\in\sigma(\D_t^i)) \\
&= \sum_{t \ge 1} 
\E^{\cpi}_t\!\left[
\mathbf{1}_{\{\tau=t\}}
\prob_t\!\left(H^\star=\Inf_t(\D_t)\mid \D_t\right)
\right].
\end{align}

The last line holds because of Lemma~\ref{lemma:posterior kernel}.
Then we introduce a Lagrange multiplier $\lambda \ge 0$. 
The corresponding Lagrangian function is given by
\[
\mathcal{L}(\cpi,\tau,I;\lambda)
=
\mathbb{E}^{\cpi}[\tau]
+
\lambda\Big((1-\delta)
-
\mathbb{P}^{\cpi}\!\left(I_\tau(\mathcal D_\tau)=H^\star\right)\Big).
\]
The dual function $g(\lambda)$ is defined as \[g(\lambda)
=
\inf_{\cpi,\tau,\Inf}
\mathcal{L}(\cpi,\tau,\Inf;\lambda).\] 
The dual problem is 
\begin{align}
    \sup_{\lambda\geq 0} g(\lambda)
&=
\sup_{\lambda\geq 0}\quad
\inf_{\cpi,\tau,\Inf}
\mathcal{L}(\cpi,\tau,\Inf;\lambda)
\end{align}

To simplify the Lagrangian function we get:
\begin{align}
\mathcal{L}(\cpi,\tau,\Inf;\lambda)
&=\mathbb{E}^{\cpi}[\tau]
+
\lambda\Big((1-\delta)
-
\mathbb{P}^{\cpi}\!\left(I_\tau(\mathcal D_\tau)=H^\star\right)\Big)\\
&=\mathbb{E}^{\cpi}[\tau]
+
\lambda\Big((1-\delta)
-
\sum_{t \ge 1} 
\E^{\cpi}\!\left[
\mathbf{1}_{\{\tau=t\}}
\prob_t\!\left(H^\star=\Inf_t(\D_t)\mid \D_t\right)
\right]\Big)\\
&=\lambda-\lambda\delta+\mathbb{E}^{\cpi}\left[\tau-\lambda\sum_{t\geq 1}
\mathbf{1}_{\{\tau=t\}} 
\prob_t\!\left(H^\star=\Inf_t(\D_t)\mid \D_t\right)
\Big)\right]\\
&=\lambda-\lambda\delta+\mathbb{E}^{\cpi}\left[\sum_{t\geq 1}\mathbf{1}_{\{t\leq\tau\}}
-\lambda\sum_{t\geq 1}
\mathbf{1}_{\{\tau=t\}} 
\prob_t\!\left(H^\star=\Inf_t(\D_t)\mid \D_t\right)
\Big)\right]\\
&=\lambda-\lambda\delta+\mathbb{E}^{\cpi}\left[\tau
-\lambda 
\prob_\tau\!\left(H^\star=\Inf_\tau(\D_\tau)\mid \D_\tau\right)
\Big)\right]
\end{align}

\end{proof}

Now we show how to embed the stopping rule into the policy by augmenting the action space as a stopping action.
At each time $t$, the joint action is $a_t = (a_t^1,\dots,a_t^N)\in \A$, where each agent chooses
$a_t^i \in \A^i$.

For each agent $i\in[N]$, define the augmented local action space
\[
\bar{\A}^i := \A^i \cup \{a_{\text{stop}}\},
\]
and the augmented joint action space
\[
\bar{\A} := \bar{\A}^1 \times \cdots \times \bar{\A}^N.
\]
We also define the induced stopping time under the augmented policy as
\[
\bar{\tau} := \inf\big\{t\ge 1:\ \exists i\in[N]\ \text{s.t.}\ a_t^i = a_{\text{stop}}\big\}.
\]

If any agent performed action  $a_t^i\neq a_{\text{stop}}$ and observed $x_{t+1}^i$ after its action decides to stop this is equivalent to choose action $a_{\text{stop}}$ at round $t+1$ leading to \(\bar{\tau}=t+1=\tau\).

Assumption~\ref{assumption:common_stopping} is what allows this construction:
Lemma~\ref{lemma:stop_action} builds an augmented policy $\{\bar\pi^i\}$
whose induced stopping time $\bar\tau$ coincides with $\tau$, which
requires $\{\tau \le t\}$ to be measurable with respect to every local
filtration. The following remark records how this condition is met in our
implementation.

\begin{remark}
\label{rem:public_termination}
Assumption~\ref{assumption:common_stopping} constrains only the measurability of $\{\tau \le t\}$, not the agents' information sets: it is
satisfied whenever the indicator $s_t = \mathbb{I}\{\exists i,\, a_t^i =
a_{\rm stop}\}$ is observed by all agents because each $a_t^i$ is selected
from $\mathcal{D}_t^i$ alone. Common stopping therefore requires each agent to transmit at most one bit per round and to learn only whether some agent has stopped; it needs no central coordinator and no sharing of observations or histories.
\end{remark}

\begin{lemma}
For every $(\cpi, \tau, \Inf)$ with $\tau \ge 1$ almost surely, there exists a policy 
$\bpi$ on $\bar{\mathcal A} = \bigcup_i (\mathcal{A}^i \cup \{a_{\text{stop}}\})$ with stopping rule
\[
\bar{\tau} := \inf\big\{t\ge 1:\ \exists i\in[N]\ \text{s.t.}\ a_t^i = a_{\text{stop}}\big\}.
\]
such that
\[
\mathcal{L}(\cpi,\tau,\Inf;\lambda)
=
\mathcal{L}(\bpi,\Inf;\lambda).
\]
\label{lemma:stop_action}
\end{lemma}
\begin{proof}

   With Assumption~\ref{assumption:common_stopping} $\{\tau \leq t\} \in \sigma(\D_t^i)\quad \forall i$, we can define for each agent  for $z^i\in \mathcal{Z}_t^i$,
  
and  there exists a measurable set 
$A_t^i \in \mathcal{B}(\mathcal{Z}_t^i)$, where 
\(\{\tau = t\} = \{\D_t^i \in A_t^i\}\)
\[
\bar{\pi}_t^i(a_{\text{stop}} \mid z^i) 
= \mathbf{1}\{z^i \in A_t^i\}
\qquad
\text{and} \qquad\bar{\pi}_t^i(a \mid z^i) 
= \pi_t^i(a \mid z^i), 
\quad a \in \mathcal{A}^i
.
\]

By the construction of the $\bpi$, $$a_\text{stop}\Longleftrightarrow \tau=t \text{ for all } i \in [N].$$
Hence, \[\bigcup_i\{a^i_t =a_\text{stop}\}=\{\tau=t\}\]
Therefore, we have
\[
 \{\bar{\tau}= t\}
= \{\exists i: a_t^i = a_{\text{stop}}\} 
=\{\tau = t\} .
\]

\begin{align}
\mathcal{L}(\bpi,\Inf;\lambda)
&=\lambda(1-\delta)+\mathbb{E}^{\bpi}\left[\bar{\tau}
-\lambda 
\prob_{\bar{\tau}}\!\left(H^\star=\Inf_{\bar{\tau}}(\D_{\bar{\tau}})\mid \D_
{\bar{\tau}}\right)
\right]\\
&=\lambda(1-\delta)+\mathbb{E}^{\bpi}\left[\sum_{t\ge 1}\mathbf{1}_{\{t\leq\bar{\tau}\}}
-\lambda 
\mathbf{1}_{\{\bar{\tau}=t\}}\prob_{t}\!\left(H^\star=\Inf_t(\D_t)\mid \D_
{t}\right)
\right] \\
&=\lambda(1-\delta)+\mathbb{E}^{\bpi}\left[\sum_{t\ge 1}^{\bar{\tau}}1
-\lambda 
\mathbf{1}_{\{\bar{\tau}=t\}}\prob_{t}\!\left(H^\star=\Inf_{t}(\D_{t})\mid \D_
{t}\right)
\right] \\
&=\lambda(1-\delta)+\mathbb{E}^{\bpi}\left[\sum_{t\ge 1}^{\bar{\tau}}1
-\lambda 
\mathbf{1}_{\{\exists i: a^i_t = a_\text{stop}\}}\prob_{t}\!\left(H^\star=\Inf_{t}(\D_{t})\mid \D_
{t}\right)
\right]\\
&=\lambda(1-\delta)+\mathbb{E}^{\cpi}\left[\sum_{t\ge 1}^{\tau}1
-\lambda 
\mathbf{1}_{\{\tau=t\}}\prob_{t}\!\left(H^\star=\Inf_{t}(\D_{t})\mid \D_{t}\right)
\right]
\end{align}
\end{proof}

Now we will keep working in the extended action space, and define \(\tau := \inf\big\{t\ge 1:\ \exists i\in[N]\ \text{s.t.}\ a_t^i = a_{\text{stop}}\big\}
\) to simply notations by avoiding using the bar notation.

\begin{proposition}
For any $t \in \mathbb{N}$, the optimal inference rule selects the hypothesis with the largest posterior probability:
\[
\Inf_t(D_t) \in \arg\max_{H \in \mathcal{H}} \mathbb{P}_t(H^\star = H \mid \D_t),
\]
Moreover,
\[
\sup_{\lambda \ge 0}\quad \inf_{\pi\in\Pi_{\text{dec}},\Inf} \mathcal{L}_\lambda(\pi,\Inf)
=
\sup_{\lambda \ge 0} \inf_{\pi}
\left\{
\lambda (1-\delta) + \mathbb{E}^{\pi}\!\left[\tau - \lambda r_\tau(\D_\tau)\right]
\right\},
\]
where
\[
\tau := \inf\big\{t\ge 1:\ \exists i\in[N]\ \text{s.t.}\ a_t^i = a_{\text{stop}}\big\}
\qquad
r_t(z_t) := \max_{H \in \mathcal{H}} \mathbb{P}_t(H^\star = H \mid \D_t=z_t).
\]
\label{prop:map-inference-dual-reduction}
\end{proposition}

\begin{proof}

Let's first optimizing over the inference rules. For \(t\in \mathbb{N},z\in \mathcal{Z}_t\). let's define \(r_t(z_t)=\max_{H\in \mathcal{H}}\prob_t(H^*=H \mid\D_t=z_t)\). If we fix \((\cpi,\tau)\), we will get 
\begin{align*}
\inf_\Inf \mathcal{L}(\cpi,\Inf;\lambda)
&=\lambda(1-\delta)+\inf_\Inf\mathbb{E}^{\cpi}\left[\sum_{t\ge 1}\mathbf{1}_{\{t\le\tau\}}
-\lambda 
\mathbf{1}_{\{\tau=t\}}\prob_{t}\!\left(H^\star=\Inf_{t}(\D_{t})\mid \D_{t}\right)
\right],\\
&=\lambda(1-\delta)+\mathbb{E}^{\cpi}\left[\inf_\Inf\sum_{t\ge 1} \mathbf{1}_{\{t\le\tau\}}
-\lambda 
\mathbf{1}_{\{\tau=t\}}\prob_{t}\!\left(H^\star=\Inf_{t}(\D_{t})\mid \D_{t}\right)
\right],\\
&=\lambda(1-\delta)+\mathbb{E}^{\cpi}\left[\sum_{t\ge 1} \mathbf{1}_{\{t\le\tau\}}
-\lambda 
\mathbf{1}_{\{\tau=t\}}\sup_{\Inf_t}\prob_{t}\!\left(H^\star=\Inf_{t}(\D_{t})\mid \D_{t}\right)
\right],\\
&=\lambda(1-\delta)+\mathbb{E}^{\cpi}\left[\sum_{t\ge 1} \mathbf{1}_{\{t\le\tau\}}
-\lambda 
\mathbf{1}_{\{\tau=t\}}r_t(\D_t)\right]\\
&=\lambda(1-\delta)+\mathbb{E}^{\cpi}\left[\tau
-\lambda 
r_t(\D_t)\right].
\end{align*}

The second-to-last equality follows from Proposition~\ref{prop:optimal inf}. Now  we can get \[
\sup_{\lambda\geq 0}\inf_{\cpi,\Inf} \mathcal{L}(\cpi,\Inf;\lambda) = \sup_{\lambda\geq 0}\inf_{\cpi} \lambda-\lambda\delta + \mathbb{E}^{\cpi}\left[\tau
-\lambda 
r_\tau(\D_\tau)\right]
\]
\end{proof}

\clearpage
We now optimize over policies. Recall we now defined $\A^i$ includes the stopping action, \(a\in\A\) and 
$\tau = \inf\{t : a_t = a_{\text{stop}}\}\Leftrightarrow\tau = \inf\big\{t\ge 1:\ \exists i\in[N]\ \text{s.t.}\ a_t^i = a_{\text{stop}}\big\}$. For $t \in \mathbb{N}$, $z \in \mathcal{Z}_t$ 
use $\pi$ as an simplified notation for any possible decentralized policy, then the optimal value to go satisfies
\[
V_t(z;\lambda) := \inf_{\pi:\tau \ge t} \lambda - \lambda\delta 
+ \mathbb{E}^{\pi}\!\left[\underbrace{(\tau - t)}_{\text{expected future sample}} - \lambda r_\tau(\D_\tau)\Biggm| \D_t = z\right].
\]

Also define the following optimal $Q$-function for $z \in \mathcal{Z}_t$, $a \neq a_{\text{stop}}$
\[
Q_t(z,a;\lambda) := 1 + \int_{\mathcal{X}} 
V_{t+1}(\underbrace{z,a,x'}_{z'};\lambda)\,\bar{P}_t(dx' \mid z,a),
\]
where $\bar{P}_t$ is the posterior predictive distribution of the next observation. Every additional step will give a cost of  1 in the Q value. We also define stopping action Q value as 
\[
Q_t(z,a_{\text{stop}};\lambda) := \lambda\bigl(1-\delta-r_t(z)\bigr).
\]
If we define $\lambda$-optimal policy \(\pi_\lambda^\star =(\pi_t^\star(z;\lambda) )_t\) as  
\[
\pi_t^\star(z;\lambda) \in \arg\min_{a\in \A} Q_t(z,a;\lambda)
\]
We have then the following result.

\begin{proposition}
    $\pi_\lambda^\star$ is a $\lambda$-optimal policy. 
Furthermore, the optimal value for $z \in \mathcal{Z}_t$ satisfies
\[
V_t(z;\lambda) = \min_{a\in\A} Q_t(z,a;\lambda). \tag{21}
\]
\label{prop:optimal_policy}
\end{proposition} 

\begin{proof}
     Fix $D_t = z$, $z \in \mathcal{Z}_t$. Assume the optimal stopping action stops at 
$\tau = t$ for such $z$. Then
\[
V_t(z;\lambda) = \lambda - \lambda\delta - \lambda r_t(z) 
= Q_t(z,a_{\text{stop}};\lambda).
\]
Otherwise, suppose the optimal stopping rule satisfies $\tau > t$, then
\begin{align*}
    V_t(z;\lambda) 
&= \inf_{\pi:\tau>t} \lambda - \lambda\delta 
+ \mathbb{E}^\pi\!\left[\tau - t - \lambda r_\tau(D_\tau)\mid D_t=z\right]\\
&= \inf_{\pi:\tau>t} \lambda - \lambda\delta 
+ 1 +\mathbb{E}^\pi\!\left[\tau - (t+1) - \lambda r_\tau(D_\tau)\mid D_{t}=z\right]\\
&=  \inf_{\pi:\tau>t} 1+\mathbb{E}^\pi\left[\mathbb{E}^\pi\left[\lambda - \lambda\delta 
+\tau - (t+1) - \lambda r_\tau(D_\tau)\mid D_{t+1}\right] \mid \D_{t}=z\right]\\
&=  \inf_{\pi:\tau>t} 1+\mathbb{E}^\pi\left[\mathbb{E}^\pi\left[V_{t+1}(\D_{t+1};\lambda) \right]\mid \D_{t}=z\right]\\
&= \inf_{\pi:\tau>t} 
1 + \mathbb{E}^\pi_a\!\left[\int_{\mathcal{X}} 
V_{t+1}(z,a,x';\lambda)\,\bar P_t(dx'|z,a)\right],
\end{align*}

We then clearly obtain the lower bound
\[
V_t(z;\lambda) \ge 
\min\!\left\{\lambda(1-\delta-r_t(z)),\;\min_{a\in A} Q_t(z,a;\lambda)\right\}.
\]

If \(\tau=t\), then \(\lambda(1-\delta-r_t(z)
)\le\min_{a\neq \text{stop}} Q_t(z,a;\lambda)\); the value to go for the optimal policy is exactly \(\lambda(1-\delta-r_t(z)
)\).

If \(\tau\neq t\), then \(\lambda(1-\delta-r_t(z)
)\ge\min_{a\neq \text{stop}} Q_t(z,a;\lambda)\) , the best policy value should go to \(Q_t(z,\pi_t(z);\lambda) = \min_{a\in \A}Q_t(z,a,\lambda) = V(z;\lambda)\).
\end{proof}

\paragraph{Reward-form dual.}
Equivalently, the constrained problem can be written in reward-maximization form as
\[
\inf_{\lambda\ge 0}\ \sup_{\pi,\Inf}\ V_\lambda(\pi,\Inf),
\]
where
\[
V_\lambda(\pi,\Inf)
:=
-\E^\pi[\tau]
+\lambda\Big(
\prob^\pi(\Inf_\tau(\D_\tau)=H^\star)-1+\delta
\Big).
\]
Optimizing over the inference rule gives
\[
r_t(\D_t):=\max_{H\in\mathcal H}\prob_t(H^\star=H\mid \D_t),
\]
and therefore
\[
\inf_{\lambda\ge0}\ \sup_{\pi}
\E^\pi\!\left[
-\tau+\lambda\bigl(r_\tau(\D_\tau)-1+\delta\bigr)
\right].
\]

We then define the reward model by
\[
r_{t,\lambda}(\D_t,a)
=
-\mathbf 1\{a\neq a_{\mathrm{stop}}\}
+
\lambda\,\mathbf 1\{a=a_{\mathrm{stop}}\}\bigl(\max_{H\in\mathcal H}\prob_t(H^\star=H\mid \D_t)\bigr).
\]
Accordingly, the $Q$-function satisfies
\[
Q_{t,\lambda}(\D_t,a)
=
r_{t,\lambda}(\D_t,a)
+
\mathbf 1\{a\neq a_{\mathrm{stop}}\}
\E_{x_{t+1}\mid(\D_t,a)}
\left[
\max_{a'} Q_{t+1,\lambda}\bigl((\D_t,a,x_{t+1}),a'\bigr)
\right].
\]

\subsection{Q-Learning Approach}
\label{sec:appendix_vdn}
From the dynamic programming formulation above, we obtain the optimal joint
action--value function associated with the optimal policy $\pi^\star$.
We denote this value by $Q_{\mathrm{tot}}$.
We next seek a decomposition of $Q_{\mathrm{tot}}$ into per-agent utilities
$Q^i$. To enable decentralized greedy action selection, we adopt the
\emph{individual--global maximization} (IGM) assumption, which ensures that
the optimal joint action can be recovered from individual greedy actions.
In particular, we assume there exists a mixing function such that
$Q_{\mathrm{tot}}$ is monotone in each $Q^i$, i.e.,
\[
\frac{\partial Q_{\mathrm{tot}}}{\partial Q^i} \ge 0, \qquad \forall i.
\]
Under this condition, decentralized greedy selection with respect to the
individual utilities is consistent with the joint optimum. In practice, we
use learned approximations $\hat Q^i$ to parameterize this decomposition,
rather than the true value functions.

Since $\tau$ is a common stopping time adapted to each local filtration in Assumption~\ref{assumption:common_stopping},
on the event $\{\tau=t\}$ all agents must stop, i.e., $a_t^i=a_{\mathrm{stop}}$ for all $i$.
Because decentralized execution is greedy with respect to the learned per-agent utilities,
we impose the following consistency constraint: on $\{\tau=t\}$, we require
\[
\Bigl(\hat Q_t^i(z_t^i,a_{\mathrm{stop}};\lambda)
\ge \hat Q_t^i(z_t^i,a;\lambda)\Bigr),
\qquad \forall a\in\A^i,\ a\neq a_{\mathrm{stop}}.
\]
Equivalently, 
\[
a_{\mathrm{stop}}\in \arg\max_{a^i\in\A^i}
\hat Q_t^i(z_t^i,a^i;\lambda).
\]
Suppose We adopt a VDN-style factorization,
\[
\hat Q_{\mathrm{tot},t}(z,\mathbf a;\lambda)=\sum_{i=1}^N \hat Q_t^i(z_t^i,a_t^i;\lambda),
\]
so that joint greedy minimization is consistent with decentralized greedy minimization.

\subsection{Policy Learning Approach}
\label{sec:policy_learn}
An alternative to Q-learning is to learn a decentralized policy using an actor--critic formulation, where the critic evaluates the actors' actions using global information and joint actions. Since execution remains decentralized, each agent \(i\) selects its action based only on its local history \(\D_t^i\). During training, however, the critic may use additional global information to improve credit assignment.

For each agent $i\in[N]$, let the decentralized actor be
\[
\pi_\theta^i(\cdot \mid \D_t^i),
\]
and \(\theta\) is the collection of actor's parameter.
\paragraph{Critic.}
For each agent $i$, we define a centralized training critic
\[
Q_t(\D_t,a_t^i,a_t^{-i};\lambda),
\]
which estimates the future soft cost when agent $i$ takes action $a_t^i$ and the other agents take $a_t^{-i}$.
If no agent stops at time $t$, we pay cost $1$ and continue. If some agent stops, we terminate with reward
\[
\lambda\bigl(r_t(\D_t)-1+\delta\bigr).
\]
Thus, the one-step target is
\[
y_t^i=
c_t
+
\gamma\,\mathbf{1}\{\forall  j,\ a_t^j\neq a_{\mathrm{stop}}\}
\Big(
Q_{\phi^{\rm tar}}(\D_{t+1},a_{t+1}^i,a_{t+1}^{-i};\lambda)
)
\Big),
\]
where
\[
c_t=
\begin{cases}
-1, & \text{if no agent stops at time } t,\\[4pt]
\lambda\bigl(r_t(\D_t)-1+\delta\bigr), & \text{otherwise.}
\end{cases}
\]
The critic loss is
\[
\mathcal{L}_{Q^i}(\phi^i)
=
\E\Big[
\big(
Q^i_{\phi^i}(\D_t,a_t^i,a_t^{-i};\lambda)-y_t^i
\big)^2
\Big].
\]

\paragraph{Actor.}
The actor is trained with a deterministic policy-gradient objective using the centralized critics. Given the local history \(\D_t^i\), each actor produces a relaxed action \(\tilde a_t^i=\pi_{\theta_i}^i(\D_t^i)\), which is used for the actor update. The objective is
\[
J(\theta)
=
\mathbb{E}\!\left[
\sum_{i=1}^N
\min_{j\in\{1,2\}}
Q_{\phi_j}\bigl(\D_t,\tilde a_t^i,\tilde a_t^{-i};\lambda\bigr)
\right],
\]
where \(\tilde a_t=(\tilde a_t^1,\ldots,\tilde a_t^N)\) denotes the joint relaxed action generated by the current actors. Since we maximize \(J(\theta)\), the actor is encouraged to choose decentralized actions that receive high value under the centralized critics. During execution, each agent still selects actions using only its local history \(\D_t^i\).

\paragraph{Stopping.}
Since $\tau$ is a common stopping time, when stopping is optimal all agents should prefer the stopping action. Hence, when \(\{t=\tau\}\) we want
\[
a_{\mathrm{stop}}
\in
\arg\max_{a^i\in\A^i}
Q_t(\D_t,a^i,a_t^{-i};\lambda),
\qquad \forall i\in[N].
\]
In the soft setting, this means that the policy should assign high probability to $a_{\mathrm{stop}}$.

\subsection{Posterior Confidence Reward for Capacity Identification}
Given a bandit problem with $K$ arms and $N$ players, the objective is identifying the true capacity vector $C^*$ up to $1-\delta$ confidence. We define the set of hypothesis $\mathcal{H}$ is all the possible capacity configuration for $N$ players and $K$ arms, $C^* = (c_i)_{i\in [K]} , c_i \in \{1,\dots, N\}$. We further define the reward in terms of the probability of correctly identifying the true capacity vector. Formally, 
\begin{equation*} 
\max_{c_{1},\dots, c_{K} \in [N]} 
\; \mathbb{P}\!\left(C^* = 
\bigl[c_{1}(\D_\tau), \dots, c_{K}(\D_\tau)\bigr]\right),
\end{equation*}
where $c_{a}$ is a random variable taking values past global history data and output the capacity estimation of arm $a$.
Essentially we can just write
\[
\max_{c_1,\dots, c_K\in [N]} 
\; \mathbb{P}\!\left(C^* = 
\bigl[c_1,\dots,c_K\bigr]\right) = \prod_{a\in[K]} \max_{c\in [N]} \mathbb{P}(C_a^\star =c),
\]

Since the capacity hypotheses across arms are independent, we can rewrite this as
\begin{equation*} 
\max_{\hat{c}_{\tau,1},\dots, \hat{c}_{\tau,K} \in [N]} 
\prod_{a \in [K]} \mathbb{P}\!\left(c_a^* = \hat{c}_{\tau,a}\right)
= \prod_{a \in [K]} \max_{\hat{c}_{\tau,a} \in [N]} 
\mathbb{P}\!\left(c_a^* = \hat{c}_{\tau,a}\right).
\end{equation*}

Taking logarithms yields
\begin{equation*} 
\sum_{a \in [K]} 
\log \Bigl(\max_{\hat{c}_{\tau,a} \in [N]} \mathbb{P}\!\left(c_a^* = \hat{c}_{\tau,a}\right)\Bigr).
\end{equation*}

Normalizing by $K$ and defining the reward from the inference network as
\begin{equation*} 
\frac{1}{K} \sum_{a \in [K]} 
\log \Bigl(\max_{j \in [N]} \mathbb{P}\!\left(c_a^* = \hat{c}_{\tau,a}\right)\Bigr).
\end{equation*}

\section{Experiment Details}
\label{sec:exp_detial_apx}
\subsection{Algorithm}
The pseudocode for ICMAPE is provided in Algorithms~\ref{alg:multicpe_td3_training} and~\ref{alg:multicpe_inference}.
\begin{algorithm}[ht]
\caption{\textsc{ICMAPE-TD3 Training}}
\label{alg:multicpe_td3_training}
\begin{algorithmic}[1]
\STATE \textbf{Input:} prior $\mathcal{P}$, confidence level $\delta$, number of agents $N$, maximum rollout length $T$.
\STATE Initialize replay buffer $\mathcal{B}$, actors $\{\pi_\theta^i\}_{i=1}^N$, inference network $\Inf_\psi$, twin critics $Q_{\phi_1},Q_{\phi_2}$, target networks, and multiplier $\zeta$.
\WHILE{training is not over}
    \STATE Sample environment $M\sim \mathcal{P}$ with hypothesis $H^\star$.
    \STATE Initialize global trajectory $\D_0$ and local trajectories $\{\D_0^i\}_{i=1}^N$.
    \STATE $t \leftarrow 0$.
    \REPEAT
        \FOR{each agent $i=1,\dots,N$}
            \STATE Select action $a_t^i \sim \pi_\theta^i(\cdot\mid \D_t^i)$ with exploration noise.
        \ENDFOR
        \STATE Execute joint action $a_t=(a_t^1,\dots,a_t^N)$ and observe feedback $x_{t+1}$.
        \STATE Update global trajectory $\D_{t+1}$ and local trajectories $\{\D_{t+1}^i\}_{i=1}^N$.
        \STATE Compute posterior $p_\psi(H\mid \D_{t+1})$ and confidence reward $r_{t+1}=\max_H p_\psi(H\mid \D_{t+1})$.
        \STATE Set $s_t=\mathbf{1}\{\exists i,\ a_t^i=a_{\mathrm{stop}}\}$ and define cost $c_t$ as in Eq.~\ref{eq:cost}.
        \STATE Store $(\D_t,a_t,x_{t+1},\D_{t+1},s_t,H^\star)$ in $\mathcal{B}$.
        \STATE $t\leftarrow t+1$.
    \UNTIL{$s_{t-1}=1$ or $t=T$}

    \STATE Sample minibatch $B\sim\mathcal{B}$.
    \STATE Update inference network $\Inf_\psi$ using $\mathcal{L}_{\mathrm{inf}}$ in Eq.~\ref{eq:inf_loss}.
    \STATE Compute TD3 targets using target actors and $\min_{\ell\in\{1,2\}} Q_{\phi_\ell'}$.
    \STATE Update critics $Q_{\phi_1},Q_{\phi_2}$ using $\mathcal{L}_{\mathrm{critic}}$ in Eq.~\ref{eq:critic_loss}.
    \IF{delayed actor update step}
        \STATE Update actors $\{\pi_\theta^i\}_{i=1}^N$ using $\mathcal{L}_{\mathrm{actor}}$ in Eq.~\ref{eq:actor_obj}.
        \STATE Apply soft updates to the target networks.
    \ENDIF
    \STATE Update multiplier $\zeta$ using the batch posterior confidence.
\ENDWHILE
\end{algorithmic}
\end{algorithm}

\begin{algorithm}[ht]
\caption{\textsc{ICMAPE Inference}}
\begin{algorithmic}[1]
\label{alg:multicpe_inference}
\STATE \textbf{Input:} trained actors $\{\pi_\theta^i\}_{i=1}^N$, trained inference network $\Inf_\psi$, unknown environment, maximum rollout length $T$.
\STATE Initialize global trajectory $\D_0$ and local trajectories $\{\D_0^i\}_{i=1}^N$
.
\STATE $t \leftarrow 0$.
\REPEAT
    \FOR{each agent $i=1,\dots,N$}
        \STATE Select action $a_t^i \sim \pi_\theta^i(\cdot\mid \D_t^i)$.
    \ENDFOR
    \STATE Execute joint action $a_t=(a_t^1,\dots,a_t^N)$ and observe feedback $x_{t+1}$.
    \STATE Update global trajectory $\D_{t+1}$ and local trajectories $\{\D_{t+1}^i\}_{i=1}^N$.
    \STATE Set $s_t=\mathbf{1}\{\exists i,\ a_t^i=a_{\mathrm{stop}}\}$.
    \STATE $t\leftarrow t+1$.
\UNTIL{$s_{t-1}=1$ or $t=T$}
\STATE Let $\tau=t$ and return
\[
\widehat H_\tau
=
\arg\max_{H\in\mathcal{H}}
p_\psi(H\mid \D_\tau).
\]
\end{algorithmic}
\end{algorithm}

\subsection{Baseline Descriptions}
\label{sec:baseline_descriptions}

\paragraph{Random Policy.}
\label{baseline:random}
In this baseline, each agent selects actions uniformly at random at every time step. A separate inference module monitors the collected data and terminates the process once the estimated probability of correctly identifying all arm capacities exceeds $1 - \delta$. This baseline serves as a lower bound on performance without coordinated exploration.

\paragraph{Independent Q-Learning.}
\label{baseline:indep}
In this baseline, each agent learns independently using an independent Q-learning framework \citep{matignon2012}. We refer to this baseline as INDEP in the experiments. Specifically, each agent maintains its own Q-network and updates it using a standard temporal-difference objective (Double DQN). As in our method, a separate inference module predicts the hypothesis, while stopping is controlled by the agents' learned stopping action.

\paragraph{Value Decomposition Networks.}
\label{baseline:vdn}
We adapt Value Decomposition Networks (VDN)~\citep{sunehag2017} to our pure exploration setting. The training objective minimizes a temporal-difference loss using the decomposed joint value function \(Q_{\mathrm{tot}}=\sum_{i=1}^N Q^i\); details are provided in Section~\ref{sec:appendix_vdn}. As above, inference and inference reward is handled by a separate module, while termination is controlled by the learned stopping action.

\paragraph{Environment Reward Guided Exploration (Bandit Only).}
\label{baseline:env_reward}
For the multi-armed bandit capacity identification task, we additionally include an environment-reward-guided exploration baseline in the A3P5 setting, where agents are trained using observed bandit rewards rather than inference confidence as the training signal. We include this baseline to demonstrate the designated failure of reward-driven exploration in pure-exploration settings, where maximizing environment rewards does not align with the goal of hypothesis identification.

\paragraph{Centralized Coordination (Bandit Only).}
\label{baseline:centralized_coord}
To quantify the gap between centralized and decentralized action selection, we added a centralized joint-policy baseline in A3P5. The centralized policy observes the full global trajectory and directly selects one of 243 joint arm assignments, plus one stopping action. We evaluated both TD3 and DQN as centralized baselines.

\paragraph{Thompson Sampling Guided Exploration (Signal Localization Only).}
\label{baseline:ts}
Thompson Sampling (TS) is only applicable when the hypothesis space aligns with the action space, as in signal-source localization, where hypotheses correspond to spatial locations that agents can move toward. In the TS baseline, each agent constructs a posterior over hypotheses from its collected data, $\mathcal{D}_t^i$, and samples a hypothesis from this posterior. The sampled hypothesis, together with a predefined action-priority rule, determines the agent's next data-collection action. Agents collect observations until the inference network stops the process. See Section~\ref{sec:agent_wise_posterior_ts} for details.

\paragraph{GP Prior Map Guided Sampling (Nitrate Only).}
\label{baseline:gp_prior}
For the nitrate concentration experiment, we include an additional baseline motivated by the availability of historical training data. Similar in spirit to the TS baseline in the signal localization setting, we sample 3{,}000 points from the training set (2001--2018) to construct a GP prior map of the concentration field. Agents then sample independently from this fixed GP prior map without active coordination or inference-driven stopping. This baseline assesses whether a static prior built from historical data suffices for identifying the highest-concentration region, or whether active multi-agent exploration provides meaningful gains.

\clearpage
\subsection{Signal Localization Environment Details}
Table~\ref{tab:signal_env_params} are the parameters we used to run the signal localization task and generate example Figure~\ref{fig:signal_map_noise}.

\begin{table}[ht]
\centering
\caption{Environment parameters for signal source localization.}
\label{tab:signal_env_params}
\begin{tabular}{ll}
\toprule
Parameter & Value \\
\midrule
$q_i$ & $0.95$ \\
$d_{\text{min}}$ & 10 \\
$\lambda$ & $0.3$ \\
$\sigma$ & $0.07$ \\
$n$ & $25$ \\
$m$ & $5$ \\
$\Delta_{\mathrm{move}}$ & $5$ \\
\bottomrule
\end{tabular}
\end{table}

\begin{figure}[ht]
    \centering
    \includegraphics[width=0.6\linewidth]{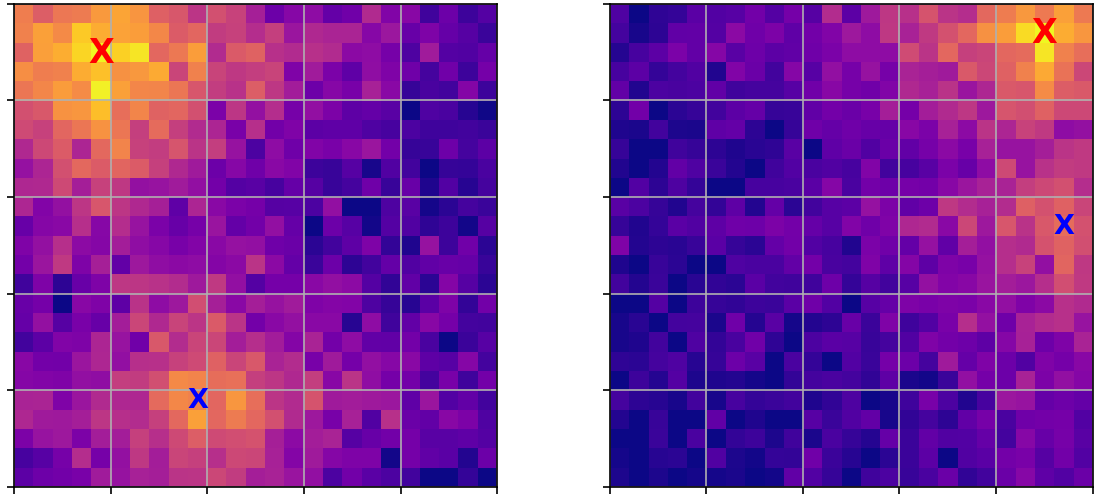}
    \caption{Example of Agent 3's signal field $Y^{(3)}$. Red $\times$: true source; blue $\times$: distractor.}
    \label{fig:signal_map_noise}
\end{figure}

\subsection{Maryland Nitrate Concentration Monitoring Details}
\label{app:nitrate}

\subsubsection{Data and Spatial Grid}

\label{app:nitrate_data_grid}

We construct the nitrate monitoring environment from Maryland nitrate concentration measurements obtained from the Water Quality Portal (WQP)~\citep{waterqualityportal2026marylandnitrate}. The dataset contains sample-level water-quality observations collected by federal, state, and local monitoring programs, including the sampling location, sampling date, measured characteristic, concentration value, and measurement unit. We use observations from 2001--2023, with years 2001--2018 used for training and years 2019--2023 held out for evaluation. To define a finite sensing domain, we discretize the Maryland region into an $n \times n$ spatial grid. Let $\mathcal{X}$ denote the set of all grid cells, with $|\mathcal{X}| = n^2$. Each cell corresponds to a candidate spatial location at which agents may collect nitrate concentration measurements. Since not all grid cells correspond to valid Maryland monitoring locations, we apply a validity mask and retain only cells that lie within the region of interest and have historical monitoring support. We denote the resulting set of valid sensing locations by $\mathcal{X}_{\mathrm{valid}} \subseteq \mathcal{X}$, with $|\mathcal{X}_{\mathrm{valid}}| = N_{\mathrm{valid}}$. Invalid cells are excluded from the agents' action space and are assigned zero concentration values when constructing spatial fields. For each year, the raw monitoring observations are used to construct an annual nitrate field over the valid grid cells. Specifically, each observation is assigned to its corresponding spatial location, and the resulting irregularly sampled measurements are used as input to the Gaussian-process field-generation procedure described in Appendix~\ref{app:nitrate_gp}. This discretized representation allows us to formulate nitrate hotspot identification as a finite pure-exploration and hypothesis testing problem, where the hypothesis set $\mathcal{H}$ is the set of valid grid cells, i.e., $\mathcal{H} = \mathcal{X}_{\mathrm{valid}}$.

\subsubsection{GP-Based Field Generation}
\label{app:nitrate_gp}

For each year, we generate spatial nitrate concentration fields using Gaussian process (GP) regression fitted to historical monitoring observations from that year. Since the raw measurements are irregularly distributed across space, the GP provides a smooth spatial interpolation over the discretized Maryland grid. Each generated field defines one latent environment instance for the active-sensing task.

At the beginning of each episode, we first sample a year from the training or evaluation split. We then draw a subset of raw nitrate measurements from that year and fit a GP in log-transformed concentration space. Specifically, if $y(x)$ denotes the nitrate concentration observed at location $x$, we fit the GP to $\log(1+y(x))$ to reduce the effect of heavy-tailed concentration values and ensure nonnegative predictions after transformation back to the original scale. The GP is fit using a Matérn kernel with fixed hyperparameters. After fitting, we evaluate the GP posterior mean and standard deviation over all grid cells in $\mathcal{X}$.

Let $\mu_y(x)$ and $\sigma_y(x)$ denote the GP posterior mean and standard deviation in log-concentration space for year $y$ at grid cell $x$. We construct the latent nitrate field by converting the posterior mean back to the original concentration scale:
\[
    f_y(x) = \exp(\max\{\mu_y(x),0\}) - 1.
\]
Invalid grid cells are masked out and assigned zero concentration. The true target location for the episode is then defined as the grid cell with the highest concentration,
\[
    \theta^\star_y = \arg\max_{x \in \mathcal{X}_{\mathrm{valid}}} f_y(x),
\]
where $\mathcal{X}_{\mathrm{valid}} \subseteq \mathcal{X}$ denotes the set of valid Maryland sensing locations.

To obtain diverse environment instances, we generate multiple GP-derived fields for each year by repeatedly subsampling the raw observations and refitting the GP. This procedure produces different plausible nitrate concentration maps from the same annual monitoring record, allowing us to evaluate whether an active-sensing policy can generalize across both years and spatial field realizations.

Figure~\ref{fig:nitrate_gp_maps} illustrates how the GP-estimated nitrate concentration fields vary across years, showing the average concentration per site for three representative years spanning the training and evaluation periods.

\begin{figure}[ht]
    \centering
    \begin{subfigure}[b]{0.8\textwidth}
        \includegraphics[width=\textwidth]{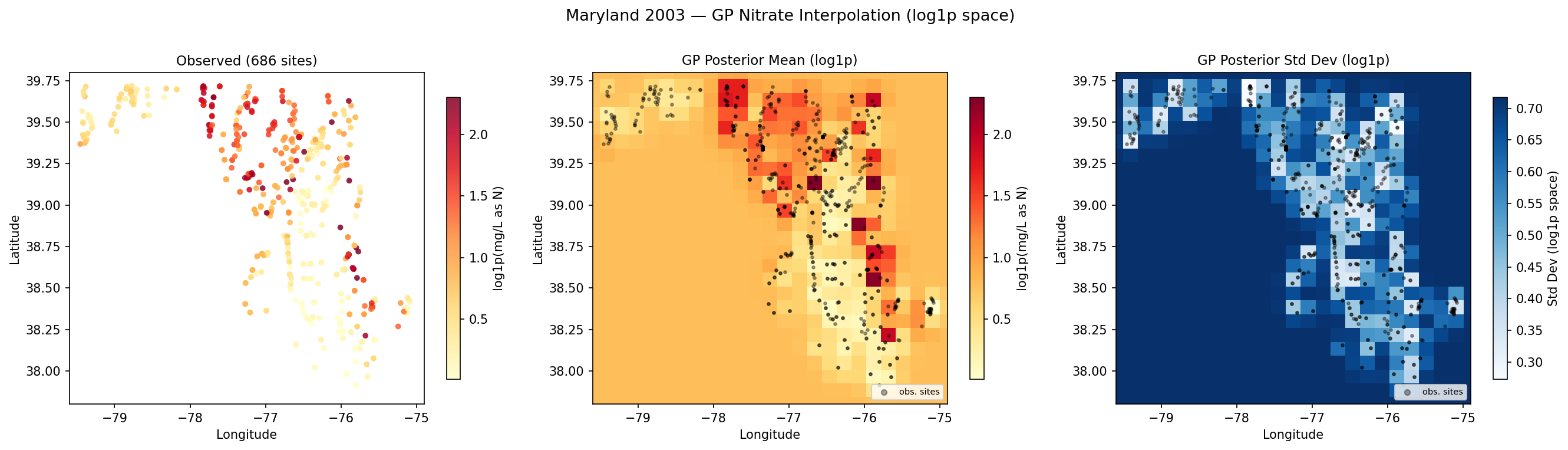}
        \caption{Year 2003 (training)}
    \end{subfigure}\\[6pt]
    \begin{subfigure}[b]{0.8\textwidth}
        \includegraphics[width=\textwidth]{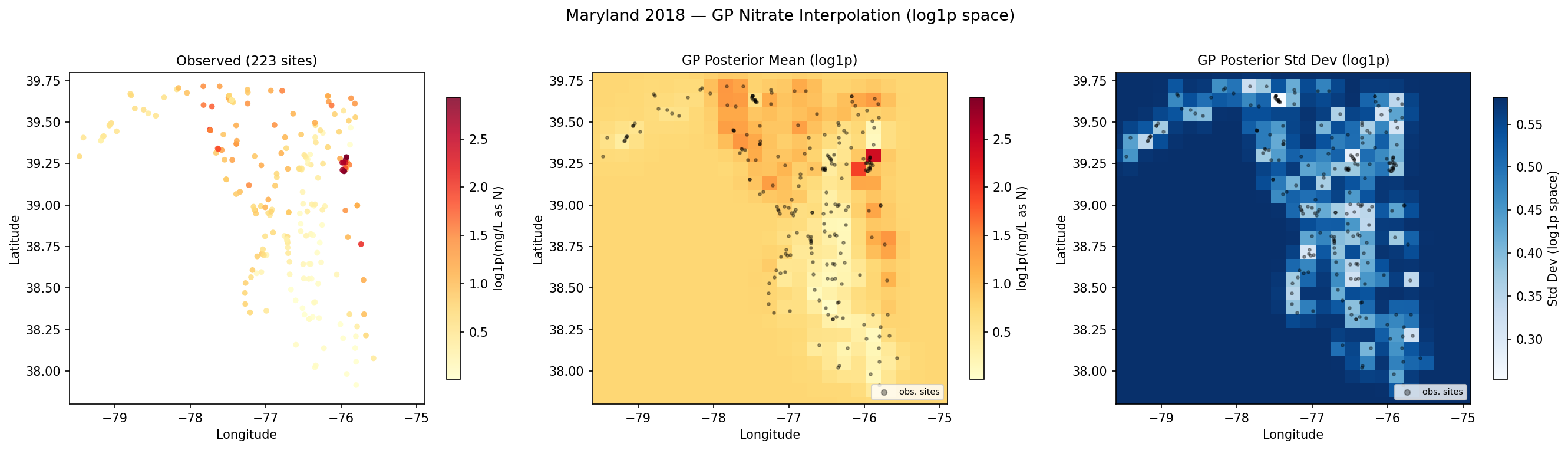}
        \caption{Year 2018 (training)}
    \end{subfigure}\\[6pt]
    \begin{subfigure}[b]{0.8\textwidth}
        \includegraphics[width=\textwidth]{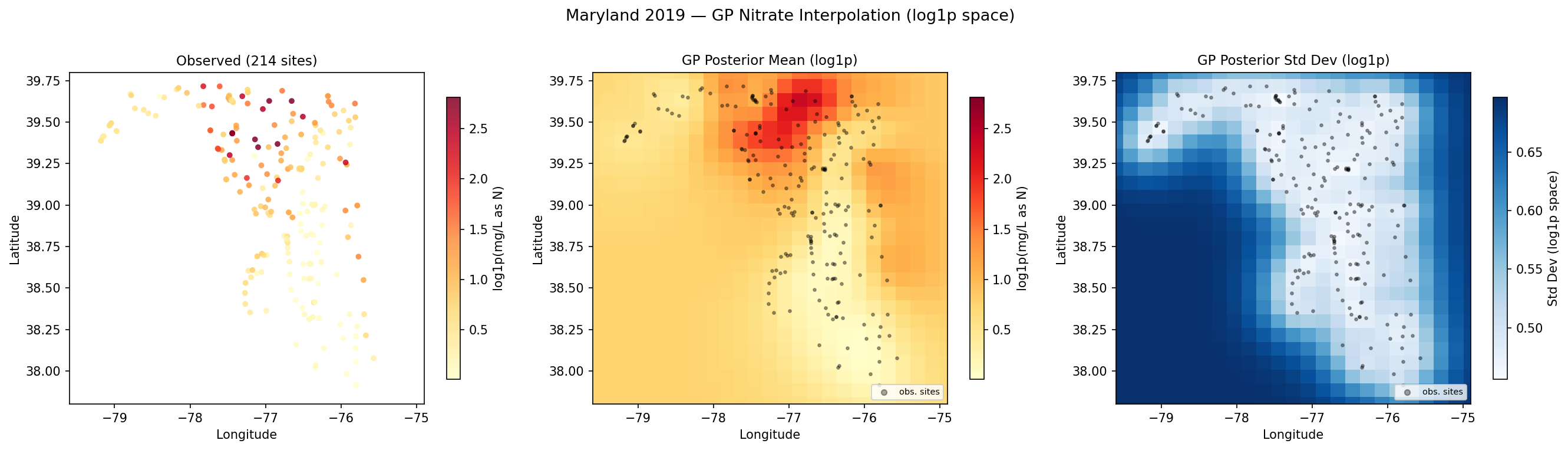}
        \caption{Year 2019 (evaluation)}
    \end{subfigure}
    \caption{GP-estimated nitrate concentration fields across Maryland for three representative years. Each map is constructed from year-round average concentrations at monitoring sites for that year. Spatial patterns shift notably across years, illustrating the inter-year variability that motivates training on multiple years and evaluating generalization on held-out years.}
    \label{fig:nitrate_gp_maps}
\end{figure}

\begin{table}[ht]
    \centering
    \small
    \begin{tabular}{rrrr}
    \toprule
    Year & Raw Data Points & Avg. Var. & Avg. Sigma \\
    \midrule
    2000 & 10,693 & 0.265 & 0.515 \\
    2001 & 10,146 & 0.287 & 0.536 \\
    2002 & 11,063 & 0.249 & 0.499 \\
    2003 & 11,542 & 0.264 & 0.514 \\
    2004 & 11,769 & 0.421 & 0.649 \\
    2005 & 10,326 & 0.238 & 0.488 \\
    2006 & 8,726  & 0.289 & 0.538 \\
    2007 & 8,294  & 0.282 & 0.531 \\
    2008 & 9,708  & 0.214 & 0.462 \\
    2009 & 8,865  & 0.219 & 0.468 \\
    2010 & 5,091  & 0.362 & 0.602 \\
    2011 & 5,532  & 0.319 & 0.565 \\
    2012 & 2,712  & 0.391 & 0.626 \\
    2013 & 3,855  & 0.317 & 0.563 \\
    2014 & 5,382  & 0.265 & 0.515 \\
    2015 & 5,651  & 0.298 & 0.546 \\
    2016 & 5,490  & 0.320 & 0.566 \\
    2017 & 1,696  & 0.533 & 0.730 \\
    2018 & 5,082  & 0.276 & 0.525 \\
    \bottomrule
    \end{tabular}
    \caption{Per-year data statistics for GP field generation. Avg. Var. denotes the mean GP posterior variance over valid Maryland grid cells in log-transformed concentration space, and Avg. Sigma is its square root.}
    \label{tab:nitrate_year_stats}
    \end{table}

\subsubsection{Multi-Agent Sampling Dynamics}
\label{app:nitrate_dynamics}

After a GP-derived nitrate field is generated, agents interact with the field through sequential sampling. We consider a team of $M$ agents operating over a finite horizon $T$. At each timestep $t$, each agent selects one valid grid cell from $\mathcal{X}_{\mathrm{valid}}$ to sample. The underlying nitrate field $f_y$ is fixed within an episode but is unknown to the agents. Agents can only obtain information about the field through noisy measurements collected at the locations they visit.

When agent $i$ selects location $x_t^i$, the environment returns noisy nitrate observations generated from the GP posterior at that cell. In our implementation, each query produces multiple noisy readings. Specifically, observations are sampled in log-concentration space as
\[
    \tilde{o}_{t,j}^i
    \sim
    \mathcal{N}\!\left(
        \mu_y(x_t^i),
        \min\{\sigma_y(x_t^i), \sigma_{\max}\}^2
    \right),
    \qquad j=1,\ldots,n_{\mathrm{obs}},
\]
where $\mu_y(x)$ and $\sigma_y(x)$ are the GP posterior mean and standard deviation for year $y$, and $\sigma_{\max}$ is a cap on the observation noise. The observations are then transformed back to the original concentration scale by
\[
    o_{t,j}^i = \exp(\tilde{o}_{t,j}^i)-1.
\]
The noise cap prevents highly uncertain grid cells from producing extremely noisy observations, which stabilizes training while preserving spatial uncertainty from the GP posterior.

Each agent maintains its own sampling history,
\[
    D_t^i = \{(x_s^i, o_{s,1}^i,\ldots,o_{s,n_{\mathrm{obs}}}^i)\}_{s=1}^{t},
\]
and the decentralized actor conditions only on this local history when selecting future actions.

The episode terminates either when the horizon $T$ is reached or when a stopping action is selected. The stopping action is shared at the team level: once any agent chooses to stop, the episode ends for all agents. The final sampling history is then passed to the inference network, which predicts the most likely high-concentration region.

\subsubsection{Target Definition and Evaluation}
\label{app:nitrate_target_eval}

For each GP-generated nitrate field $f_y$, we define the target as the set of top-$k$ highest-concentration valid grid cells,
\(
    \mathcal{H}^\star_y
\).
We use a top-$k$ target rather than only the single maximum cell because several high-concentration locations may have very similar nitrate levels, making the exact top-1 location sensitive to small GP or observation noise. At the end of each episode, the inference network outputs a confidence score $\widehat{p}(x)$ for each valid cell, and the predicted location is $\widehat{\theta} = \arg\max_{x \in \mathcal{X}_{\mathrm{valid}}} \widehat{p}(x)$. An episode is counted as successful if $\widehat{\theta} \in \mathcal{H}^\star_y$. We report top-$k$ identification accuracy averaged over generated fields and evaluation years, together with the average number of samples used before termination to measure sample efficiency.

Table~\ref{tab:nitrate_env_params} summarizes the main variables and parameters used in the nitrate monitoring enviroment.
\begin{table}[ht]
    \centering
    \small
    \begin{tabular}{lll}
    \toprule
    Symbol / Parameter & Value & Description \\
    \midrule
    $\mathcal{X}$ & $n \times n$ grid & Discretized Maryland sensing domain \\
    $n$ & $20$ & Grid resolution along each spatial dimension \\
    $\mathcal{X}_{\mathrm{valid}}$ & valid grid cells & Cells inside the region of interest with historical monitoring support \\
    $N_{\mathrm{valid}}$ & $|\mathcal{X}_{\mathrm{valid}}|$ & Number of valid sensing locations \\
    $M$ & $6$ & Number of agents \\
    $T$ & $40$ & Maximum episode horizon \\
    $k$ & $5$ & Number of high-concentration target cells \\
    $n_{\mathrm{obs}}$ & $5$ & Number of noisy readings returned per query \\
    $\mathcal{Y}_{\mathrm{train}}$ & 2001--2018 & Training years \\
    $\mathcal{Y}_{\mathrm{test}}$ & 2019--2023 & Held-out evaluation years \\
    $N_{\mathrm{fit}}$ & $100$ & Number of raw monitoring observations used to fit each training GP field \\
    $f_y$ & GP-derived field & Latent nitrate concentration field for year $y$ \\
    $\mathcal{H}^\star_y$ & top-$k$ cells & True high-concentration hotspot set for year $y$ \\
    $\sigma_{\max}$ & $0.5$ & Cap on GP posterior standard deviation in log space \\
    $N_{\mathrm{GP}}$ & $100$ & Number of GP-generated maps sampled per year for evaluation \\
    $N_{\mathrm{points}}$ & $1000$ & Number of spatial points sampled from each GP-generated map for evaluation \\
    $N_{\mathrm{train}}$ & $400{,}000$ & Maximum number of training steps \\
    \bottomrule
    \end{tabular}
    \caption{Main variables and parameters used in the Maryland nitrate concentration monitoring environment.}
    \label{tab:nitrate_env_params}
    \end{table}

\subsection{Model Architecture}
\label{sec:architecture_appendix}
\subsubsection{Capacity Identification Model Architecture}
\label{sec:capacity_detail_appendix}

For the bandit capacity identification task, the model operates on sequential trajectory data collected by the agents.. The architecture is based entirely on sequence encoders that process the history of observed actions and rewards. We use the LSTM-based encoder with conditional pooling described in Section~\ref{sec:lstm_conditional_pool} to summarize the trajectory up to the current timestep.

\paragraph{Actor Network.}
The actor network parameterizes the decentralized exploration policy for each agent. Given an agent's local trajectory $\D_t^i$, the input sequence is first mapped into an embedding space through a linear layer followed by a GELU activation. The embedded sequence is then passed through the LSTM sequence encoder described in Section~\ref{sec:lstm_conditional_pool}.

Let $\widetilde{h}_t^i$ denote the pooled trajectory representation for agent $i$. The actor maps this representation to logits over the action space with a linear layer.
where $A$ is the number of bandit arms and the additional action corresponds to the stopping action $a_{\mathrm{stop}}$.
During training, we use Gumbel--Softmax sampling with temperature $\tau$ to obtain differentiable action representations for TD3-style policy optimization. At execution time, each agent acts using only its own local trajectory $\D_t^i$, preserving decentralized execution.

\paragraph{Critic Network.}
The critic is implemented as a twin LSTM-based critic. The global trajectory $\D_t$ is embedded with a linear layer followed by a GELU activation, processed by an LSTM encoder, and pooled using the attention-based pooling layer from Section~\ref{sec:lstm_conditional_pool}. The critic also conditions on the joint action through a masked one-hot encoding: for each queried agent $i$, the agent's own action is masked out, leaving only the actions of the other agents, denoted by $\widetilde{a}_t^{-i}$. The pooled trajectory representation $\widetilde{h}_t$ and the masked action representation are concatenated and passed through a two-layer MLP to produce the Q-value estimate:
\[
Q_{\phi_\ell}^i(\D_t,a_t^{-i})
=
f_{\phi_\ell}\bigl([\widetilde{h}_t;\widetilde{a}_t^{-i}]\bigr),
\qquad \ell\in\{1,2\}.
\]
The two critics use separate LSTM encoders and separate output heads, and TD3 targets are computed using the minimum of the two target critic values.

\paragraph{Inference Network.}
The inference network predicts the underlying capacity vector from the global trajectory $\D_t$. It uses the same LSTM-based sequence encoder described in Section~\ref{sec:lstm_conditional_pool}. The input trajectory is first embedded through a linear layer followed by a GELU activation. The embedded sequence is then processed by the LSTM encoder, and attention-based pooling with a timestep mask is used to obtain a fixed-dimensional trajectory representation $\widetilde{h}_t$.

The pooled representation is mapped through a linear output layer to produce logits over possible capacities for each arm:
\[
o_\psi(\D_t)
=
W_{\mathrm{out}}\widetilde{h}_t+b_{\mathrm{out}},
\qquad
o_\psi(\D_t)\in\mathbb{R}^{A\times N},
\]
where $A$ is the number of arms and $N$ is the number of players. 
During training, the inference network is supervised using the true capacity vector $C^\star$.

\subsubsection{Signal Localization Model Architecture}
\label{sec:signal_loc_detail_appendix}

The signal localization model uses a different architecture from the bandit model because the input contains both spatial observations and temporal location histories. Specifically, each agent observes a local belief map on an $n\times n$ grid, while the model also receives the location history of all agents over time. We process these two sources of information separately and then fuse them to produce each agent's action distribution.

\paragraph{Actor Network.}
The actor is implemented as a per-agent CNN--LSTM architecture. For each agent $i$, the local belief map is first processed by a player-specific CNN encoder, producing a spatial embedding $z_{\mathrm{grid}}^i \in \mathbb{R}^{d}$. In parallel, the full location history of all agents is reshaped into a sequence of joint locations and passed through a player-specific LSTM, producing a temporal location embedding $z_{\mathrm{loc}}^i \in \mathbb{R}^{d}$ at the current timestep using LSTM with conditional pooling mentioned in Section~\ref{sec:lstm_conditional_pool}.

The two embeddings are then concatenated and passed through a player-specific fusion module, implemented as a linear projection followed by a GELU activation.
Finally, a player-specific output layer maps $z_{\mathrm{fuse}}^i$ to logits over $A+1$ actions, where the additional action corresponds to the stopping action. The policy distribution is obtained by applying a softmax to these logits. During TD3 training, we use Gumbel--Softmax samples to obtain differentiable action representations for the actor update.

\paragraph{Critic Network.}
The critic for the signal localization environment is implemented as a twin centralized critic. Unlike the actor, which conditions only on each agent's local trajectory, the critic has access to global trajectory information during training. Specifically, it takes as input the belief maps for all agents, the joint location history, and the joint action.

For each critic head, the stacked belief maps are first processed by a CNN encoder. Since the critic observes the belief maps of all agents, the input is represented as a multi-channel grid, where the channels correspond to the agents' belief maps. The CNN encoder produces a spatial embedding $z_{\mathrm{grid}} \in \mathbb{R}^d$. In parallel, the location history of all agents is reshaped into a sequence of joint locations and passed through an LSTM, producing a temporal embedding $z_{\mathrm{loc}} \in \mathbb{R}^d$.

To construct an agent-specific critic input, the spatial and temporal embeddings are shared across agents and concatenated with a one-hot encoding of the queried agent identity. The critic also receives an action feature vector encoding the actions of the other agents. Following our bandit critic design, we mask out the queried agent's own action so that the critic for agent $i$ conditions on $a^{-i}$ rather than directly on $a^i$.

The resulting critic representation is
\[
z_{\mathrm{critic}}^i
=
\left[
z_{\mathrm{grid}};
z_{\mathrm{loc}};
e_i;
a^{-i}
\right],
\]
where $e_i$ is the one-hot encoding of agent $i$. This representation is passed through a feed-forward fusion layer followed by a ReLU activation, and then through a linear output layer to produce the Q-value estimate for each agent.

To reduce overestimation bias, we use two independent critic heads, as in TD3. Each critic head has its own CNN encoder, location LSTM, fusion layer, and output layer. Thus, the critic outputs
\[
\{Q_{\phi_1}^i, Q_{\phi_2}^i\}_{i=1}^N,
\]
which are used to compute TD3 targets by taking the minimum over the two critic estimates.

\paragraph{Inference Network.}
The inference network for signal localization predicts the hidden source location from the aggregated grid representation. Unlike the actor and critic, which use both spatial observations and location histories, the inference network is implemented as a fully convolutional network over the grid input. Given a multi-channel grid $G_t \in \mathbb{R}^{C \times n \times n}$ constructed from the global trajectory $\D_t$, the network outputs logits over all spatial locations:
\[
\ell_\psi(G_t) \in \mathbb{R}^{n \times n}.
\]
The architecture consists of two $3\times 3$ convolutional layers with ReLU activations, followed by a $1\times 1$ convolution that maps the hidden feature maps to a single logit per grid cell.

The posterior distribution over candidate source locations is obtained by applying a softmax over the flattened grid:
\[
p_\psi(H=(u,v)\mid \D_t)
=
\frac{
\exp(\ell_\psi(G_t)_{u,v})
}{
\sum_{u'=1}^{n}\sum_{v'=1}^{n}
\exp(\ell_\psi(G_t)_{u',v'})
}.
\]
The final estimate is given by the maximum a posteriori prediction:
\[
\widehat H_t
=
\arg\max_{(u,v)}
p_\psi(H=(u,v)\mid \D_t).
\]
During training, the inference network is supervised using the true source location $H^\star$, and its posterior confidence is used to define the learning signal for the exploration policy.

\subsection{Nitrate Concentration Architecture}

For the nitrate monitoring task, we use a CNN-LSTM architecture for both the centralized critic and the inference network. The main difficulty is that useful information is naturally organized by spatial cell, while the raw sampling history is organized as a sequence of agent-time observations. For example, observations from the same cell may appear at different timesteps and may be collected by different agents. A standard LSTM or Transformer over the raw sequence must learn to associate these scattered observations with the same spatial cell and aggregate them across both time and agents, which can be inefficient and unstable. The CNN-LSTM architecture addresses this by first constructing cell-level spatial summaries and then using a convolutional encoder to extract spatial features.

\subsubsection{CNN-LSTM Architecture}
\label{app:nitrate_cnn_lstm}

To address this, we first convert the joint sampling history into grid-level summary statistics. For each grid cell, we aggregate all observations collected at that location across agents and timesteps, including statistics such as the normalized mean, standard deviation, maximum observed value, sampling count, and whether the cell has been sampled. This produces a structured $n \times n$ spatial representation of the team's collected information. A CNN encoder is then applied to this grid tensor to extract spatial features that capture local concentration patterns and sampling coverage.

In parallel, an LSTM encoder processes each agent's action history to summarize temporal sampling behavior. The CNN-based spatial representation and the LSTM-based history representation are then combined to form a global representation of the current sampling state. In the centralized critic, this representation is used to output action-value estimates for each agent. In the inference network, the same type of CNN-LSTM backbone is used to predict confidence scores over the valid grid cells, indicating which locations are likely to belong to the high-concentration target set.

We use the CNN-LSTM architecture for the critic and inference network, where centralized training permits access to the full joint sampling history. During execution, the actor remains decentralized and conditions only on each agent's local sampling history.

\subsection{Recurrent Sequence Encoder with Conditioned Pooling}
\label{sec:lstm_conditional_pool}
Each network in our framework takes as input a trajectory $x_{1:t}$ such as observations or state--action sequences. At a given timestep $t$, the sequence  $x_{1:t}$ is first mapped into an embedding space using a linear layer, and then processed by a recurrent encoder:
\[
h_{1:t} = \lstm(\emb(x_{1:t})).
\]

Depending on the environment, this encoder can either be shared across agents when agents are homogeneous; or instantiated separately when agents have different observation spaces. In either case, the overall architecture remains the same. To summarize the sequence of hidden state into a fixed-dimensional representation, we use a \emph{conditioned temporal pooling} operator, which lets the model adaptively focus on relevant parts of the trajectory.

\paragraph{Conditioned Temporal Pooling}

We define the pooled representation at time $t$ as
\(
c_t = \sum_{s=1}^t \alpha_s h_s,
\)
where the weights $\alpha_s$ are computed using a learned query-key interaction
\(
\alpha_s \propto \exp\!\big(\frac{\langle q_{\text{in}}, W_k h_s \rangle}{\sqrt{d}}\big).
\)
Here, $q_{\text{in}}$ is a learned query vector and $W_k$ is a linear projection. In practice, sequences are padded to a fixed horizon, and we apply a mask during pooling so that the representation at time $t$ only uses information from steps $1$ through $t$, ignoring padded entries and preventing leakage from future timesteps. The resulting context vector $c_t$ is then passed through a gating mechanism, 
\(
\tilde{h}_t = c_t \odot \sigma(W_g q_{\text{out}}),
\)
where $q_{\text{out}}$ is another learned query vector and $\sigma(\cdot)$ is a nonlinear activation function. This conditioned pooling allows the model to emphasize the most informative parts of the trajectory, rather than relying only on the final hidden state or uniformly averaging over timesteps. Each network (actor, critic, and inference) uses its own instance of the sequence encoder described above, sharing only the same architecture but not parameters.

\subsection{Posterior Thompson Sampling Baseline}

\subsubsection{Agent-wise Posterior Construction}
\label{sec:agent_wise_posterior_ts}

We describe the posterior used by each agent in the posterior Thompson-sampling baseline. 
Each agent constructs a belief using its own observations and the observed trajectories of the other agents. 
Let $s^\star\in[n]\times[n]$ denote the true source location, and let $d^\star\in[n]\times[n]$ denote the distractor location. 
For agent $i$, let $u_t^i$ denote its position at time $t$, and let $\mathcal V_t^i=\{u_\tau^i:\tau\le t\}$ denote the cells visited by agent $i$ up to time $t$.

\paragraph{Coverage likelihood.}
We first define the negative-evidence contribution induced by an agent's trajectory. 
For an agent $j\in\{1,2\}$ that can miss the true source with probability $\delta_j$, observing that agent $j$ visited a cell $u_\tau^j$ without identifying the source downweights the hypothesis that $s^\star=u_\tau^j$. 
We write this coverage likelihood as
\[
\ell_j(s,u_\tau^j)
=
\begin{cases}
\delta_j, & u_\tau^j=s,\\
1, & u_\tau^j\neq s.
\end{cases}
\]
The cumulative coverage likelihood from a set of agents $\mathcal J$ is
\[
L_{\mathcal J}(s)
=
\prod_{j\in\mathcal J}
\prod_{\tau\le t}
\ell_j(s,u_\tau^j).
\]

\paragraph{Posteriors for agents 1 and 2.}
Agents $1$ and $2$ are binary detectors of the true source. 
For $i\in\{1,2\}$, agent $i$ maintains a posterior over the true source location,
\[
p_t^i(s)
=
\mathbb P(s^\star=s\mid \mathcal D_t^i,\mathcal V_t^{-i}),
\]
where $\mathcal D_t^i$ is agent $i$'s local binary observation history and $\mathcal V_t^{-i}$ denotes the observed trajectories of the other agents. 
The prior is initialized uniformly:
\[
p_0^i(s)=\frac{1}{n^2},
\qquad s\in[n]\times[n].
\]

Let $y_t^i\in\{0,1\}$ denote the binary observation of agent $i$. 
If agent $i$ visits the true source cell, it misses the target with probability $\delta_i$. 
Thus, if $u_t^i=s$,
\[
P(y_t^i=1\mid s)=1-\delta_i,
\qquad
P(y_t^i=0\mid s)=\delta_i.
\]
If $u_t^i\neq s$, we assume a small false-positive probability $\epsilon_i$ (where $\epsilon_{1,2}=0$):
\[
P(y_t^i=1\mid s)=\epsilon_i,
\qquad
P(y_t^i=0\mid s)=1-\epsilon_i.
\]

Agent $i$ also uses the trajectories of the other binary detector and agent $3$ as negative evidence. 
Therefore, for $i\in\{1,2\}$, define
\[
\mathcal J_i=\{1,2,3\}\setminus\{i\}.
\]
The posterior used by agent $i$ is
\[
p_t^i(s)
=
\frac{
p_0^i(s)
\left[
\prod_{\tau\le t}P(y_\tau^i\mid s)
\right]
L_{\mathcal J_i}(s)
}{
\sum_{s'}
p_0^i(s')
\left[
\prod_{\tau\le t}P(y_\tau^i\mid s')
\right]
L_{\mathcal J_i}(s')
}.
\]
Equivalently, the recursive update is
\[
p_t^i(s)
\propto
p_{t-1}^i(s)
P(y_t^i\mid s)
\prod_{j\in\mathcal J_i}\ell_j(s,u_t^j),
\qquad i\in\{1,2\}.
\]
This posterior downweights cells that were visited by other agents without source identification, while still using agent $i$'s own binary observations as direct evidence.

\paragraph{Posterior for agent 3.}
Agent $3$ observes a continuous signal field affected by both the true source and an unknown distractor. 
Therefore, agent $3$ maintains a joint posterior over source-distractor pairs:
\[
p_t^3(s,d)
=
\mathbb P(s^\star=s,d^\star=d\mid \mathcal D_t^3,\mathcal V_t^1,\mathcal V_t^2).
\]
The prior is initialized uniformly over valid source-distractor pairs:
\[
p_0^3(s,d)\propto 1,
\qquad s\neq d,
\]
and
\[
p_0^3(s,d)=0
\qquad \text{if } s=d.
\]

For a candidate pair $(s,d)$, the expected signal at cell $u$ is
\[
m_{s,d}(u)
=
\max\left\{
\exp\!\left(-\lambda \operatorname{dist}(u,s)\right),
0.7\exp\!\left(-\lambda \operatorname{dist}(u,d)\right)
\right\}.
\]
Here, the first term corresponds to the true source and the second term corresponds to the distractor with intensity $0.7$. 
Let $W_t^3$ denote the local observation window of agent $3$ at time $t$. 
For each cell $u\in W_t^3$, agent $3$ observes a continuous value $y_t^3(u)$, modeled as
\[
y_t^3(u)\mid s,d
\sim
\mathcal N\!\left(m_{s,d}(u),\sigma^2\right).
\]
Assuming conditionally independent observation noise across cells in the local window, the likelihood of agent $3$'s observation is
\[
P(y_t^3\mid s,d)
=
\prod_{u\in W_t^3}
\frac{1}{\sqrt{2\pi\sigma^2}}
\exp\left(
-\frac{\left(y_t^3(u)-m_{s,d}(u)\right)^2}{2\sigma^2}
\right).
\]

Agent $3$ also uses the observed trajectories of agents $1$ and $2$ as negative evidence. 
Since agents $1$ and $2$ miss the true source with probabilities $\delta_1$ and $\delta_2$, respectively, their coverage likelihood is
\[
L_{\{1,2\}}(s)
=
\prod_{j=1}^2
\prod_{\tau\le t}
\ell_j(s,u_\tau^j).
\]
Combining agent $3$'s continuous signal likelihood with this coverage information, the posterior for agent $3$ is
\[
p_t^3(s,d)
=
\frac{
p_0^3(s,d)
\left[
\prod_{\tau\le t}P(y_\tau^3\mid s,d)
\right]
L_{\{1,2\}}(s)
}{
\sum_{s'}\sum_{d'}
p_0^3(s',d')
\left[
\prod_{\tau\le t}P(y_\tau^3\mid s',d')
\right]
L_{\{1,2\}}(s')
}.
\]
Equivalently, the recursive update is
\[
p_t^3(s,d)
\propto
p_{t-1}^3(s,d)
P(y_t^3\mid s,d)
\ell_1(s,u_t^1)
\ell_2(s,u_t^2).
\]

Since the final goal is to identify the true source rather than the distractor, agent $3$ marginalizes over the distractor:
\[
p_t^3(s)
=
\sum_d p_t^3(s,d).
\]

\paragraph{Posterior sampling.}
After updating its posterior, each agent samples a candidate true source location. 
For agents $1$ and $2$,
\[
\tilde{s}_t^i\sim p_t^i(s),
\qquad i\in\{1,2\}.
\]
For agent $3$,
\[
\tilde{s}_t^3\sim p_t^3(s),
\qquad
p_t^3(s)=\sum_d p_t^3(s,d).
\]
The sampled location $\tilde{s}_t^i$ is then used to determine the next movement action of agent $i$.

\subsubsection{Deterministic Movement Rule After Posterior Sampling}
\label{sec:deterministic_movement_ts}

Because agents cannot directly query arbitrary cells, each agent moves one step toward its sampled source using a fixed priority order.

Agent $1$ uses the order
\[
(\mathrm{down},\mathrm{left},\mathrm{up},\mathrm{right}),
\]
agent $2$ uses
\[
(\mathrm{right},\mathrm{down},\mathrm{up},\mathrm{left}),
\]
and agent $3$ uses
\[
(\mathrm{right},\mathrm{up},\mathrm{left},\mathrm{down}).
\]
Let $\operatorname{Move}(p,a)$ be the position obtained by taking action $a$ from position $p$, with boundary clipping. 
Agent $i$ selects the first action $a$ in its priority list such that
\[
\operatorname{dist}(\operatorname{Move}(p_t^i,a),\tilde{s}_t^i)
<
\operatorname{dist}(p_t^i,\tilde{s}_t^i).
\]
If no action decreases the distance, the agent stays. 
The different priority orders break symmetry and encourage different movement patterns across agents.

\subsubsection{Stopping Rule}
\label{sec:posterior_ts_stopping}

Action selection is decentralized and posterior-based, but stopping uses the same aggregated inference module as the main method. 
At each time $t$, the global trajectory $\mathcal D_t$ is passed to the inference module, which outputs a posterior $q_\psi(s\mid\mathcal D_t)$ over source locations. 
The inferred source is $\hat{s}_t=\arg\max_s q_\psi(s\mid\mathcal D_t)$, and the confidence is $r_t=\max_s q_\psi(s\mid\mathcal D_t)$.

The baseline stops at $\tau=\inf\{t:\max_s q_\psi(s\mid\mathcal D_t)\ge 1-\delta\}$ and outputs $\hat{s}_\tau=\arg\max_s q_\psi(s\mid\mathcal D_\tau)$. 
Thus, this baseline differs from the learned policy in the action-selection rule.

\subsection{Implementation Details}
\label{sec:Implementation Details apx}
\paragraph{Seed Option}
All experiments are conducted using three randomly selected seeds: 215, 219, and 321.

\paragraph{Software dependencies.}
All experiments are implemented in Python. We use PyTorch~\citep{paszke2019pytorch} for neural network training, Gymnasium~\citep{towers2024gymnasium} for environment implementation, Hydra~\citep{Yadan2019Hydra} for configuration management, and Weights \& Biases~\citep{wandb} for experiment tracking and visualization. Data processing and analysis are performed using NumPy~\citep{harris2020array}, Pandas~\citep{mckinney2010data}, Matplotlib~\citep{hunter2007matplotlib}, and Seaborn~\citep{waskom2021seaborn}. The main package versions used in our experiments are:
\[
\texttt{numpy}=2.2.6,\quad
\texttt{pandas}=2.3.1,\quad
\texttt{gymnasium}\geq 1.0.0,\quad
\texttt{torch}=2.7.1,\quad
\texttt{matplotlib}=3.10.
\]

\paragraph{Hardware.}
All experiments are run on NVIDIA L40S GPUs with 48GB of GPU memory. Capacity identification experiments take approximately 60 hours to complete, while signal localization experiments finish within approximately 15 hours. The observed GPU utilization during training is typically around 40--50\%.

\section{Additional Experiment Analysis Detailes}
\label{sec:additional_exp_apx}

\subsection{Capacity Identification.}
\subsubsection{Training Dynamics}
We set the confidence level to $\delta = 0.1$, corresponding to a target accuracy of $1-\delta = 0.9$.  In the A3P5 setting, both TD3 and VDN are able to reach the target accuracy early in training. However, once the reward structure begins to favor shorter trajectories more strongly, TD3 adapts more effectively by reducing the average stopping time.  This suggests that TD3 learns an effective sampling strategy that reduces stopping time after reaching the target accuracy. In contrast, VDN does not exhibit a similar stable reduction, resulting in higher stopping times despite achieving comparable accuracy.

In the A5P3 setting, the capacity identification problem becomes more challenging due to the larger hypothesis space. In this case, only TD3 is able to reliably reach the target accuracy and improve the sampling strategy. 
\begin{figure}[ht]
    \centering
        \includegraphics[width=\linewidth]{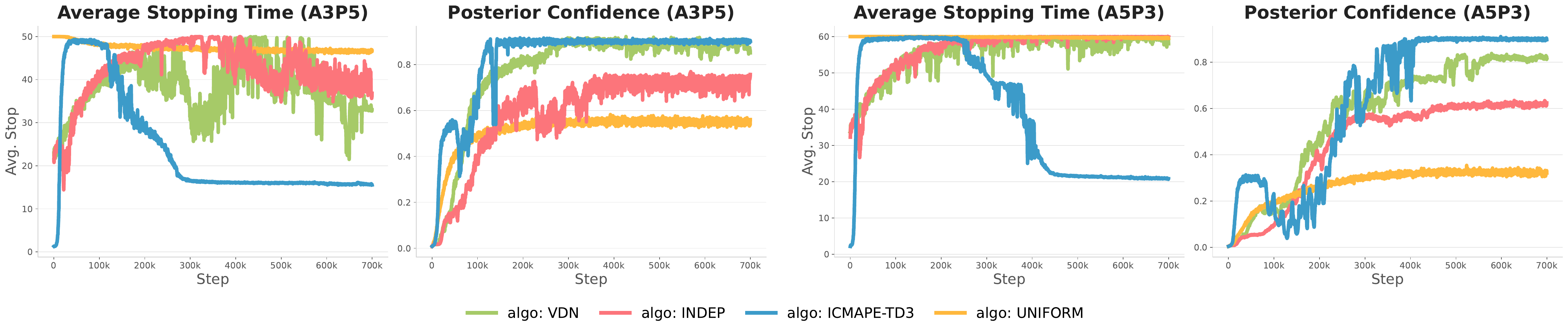}
    \caption{Training dynamics in A3P5 and A5P3.
        For each setting, we report the average stopping time and posterior confidence.
        Curves are averaged over three random seeds and smoothed with an exponential moving average parameter of $0.2$.
    }
    \label{fig:training_dynamics_bandit_appendix}
\end{figure}
\FloatBarrier

\subsubsection{Posterior Calibration}
\label{sec:posterior_cal}
To assess how well the inference module's output probabilities align with
empirical accuracy, we compute the Expected Calibration Error \citep{naeini2015obtaining} on
A3P5 over $10{,}000$ trajectories,  pooling the posterior predictions from all time steps, obtaining $\mathrm{ECE} = 0.0065$. The gap between mean
predicted confidence and empirical accuracy is at most $0.02$ in every bin,
as reported in Table~\ref{tab:calibration}.

\begin{table}[t]
  \caption{Posterior calibration in A3P5. Predicted confidence closely tracks
  empirical accuracy across all ten bins.}
  \label{tab:calibration}
  \centering
  \small
  \begin{tabular}{lcccr}
    \toprule
    Confidence Bin & Mean Confidence & Empirical Accuracy & $|\Delta|$ & Samples ($n$) \\
    \midrule
    $[0.0, 0.1)$ & 0.03 & 0.03 & 0.00 & 69{,}777 \\
    $[0.1, 0.2)$ & 0.14 & 0.13 & 0.01 & 13{,}284 \\
    $[0.2, 0.3)$ & 0.23 & 0.22 & 0.01 & 8{,}793 \\
    $[0.3, 0.4)$ & 0.34 & 0.33 & 0.01 & 6{,}280 \\
    $[0.4, 0.5)$ & 0.46 & 0.44 & 0.02 & 8{,}143 \\
    $[0.5, 0.6)$ & 0.53 & 0.53 & 0.00 & 1{,}127 \\
    $[0.6, 0.7)$ & 0.66 & 0.68 & 0.02 & 2{,}599 \\
    $[0.7, 0.8)$ & 0.76 & 0.74 & 0.02 & 5{,}806 \\
    $[0.8, 0.9)$ & 0.85 & 0.85 & 0.00 & 10{,}361 \\
    $[0.9, 1.0)$ & 0.95 & 0.95 & 0.00 & 8{,}535 \\
    \bottomrule
  \end{tabular}
\end{table}

\FloatBarrier
\subsubsection{Additional Analysis}

We further analyze the stopping behavior of different methods using survival plots for the bandit environments (Figure~\ref{fig:survival_bandit}). In both the 3-arm 5-player (A3P5) and 5-arm 3-player (A5P3) settings, the TD3-based architecture exhibits a clear exponential decay in survival probability, indicating that the model learns to terminate exploration efficiently once sufficient confidence is achieved. The plots report the average survival probability across three random seeds for each method.

\begin{figure}[ht]
    \centering
    \includegraphics[width=0.8\linewidth]{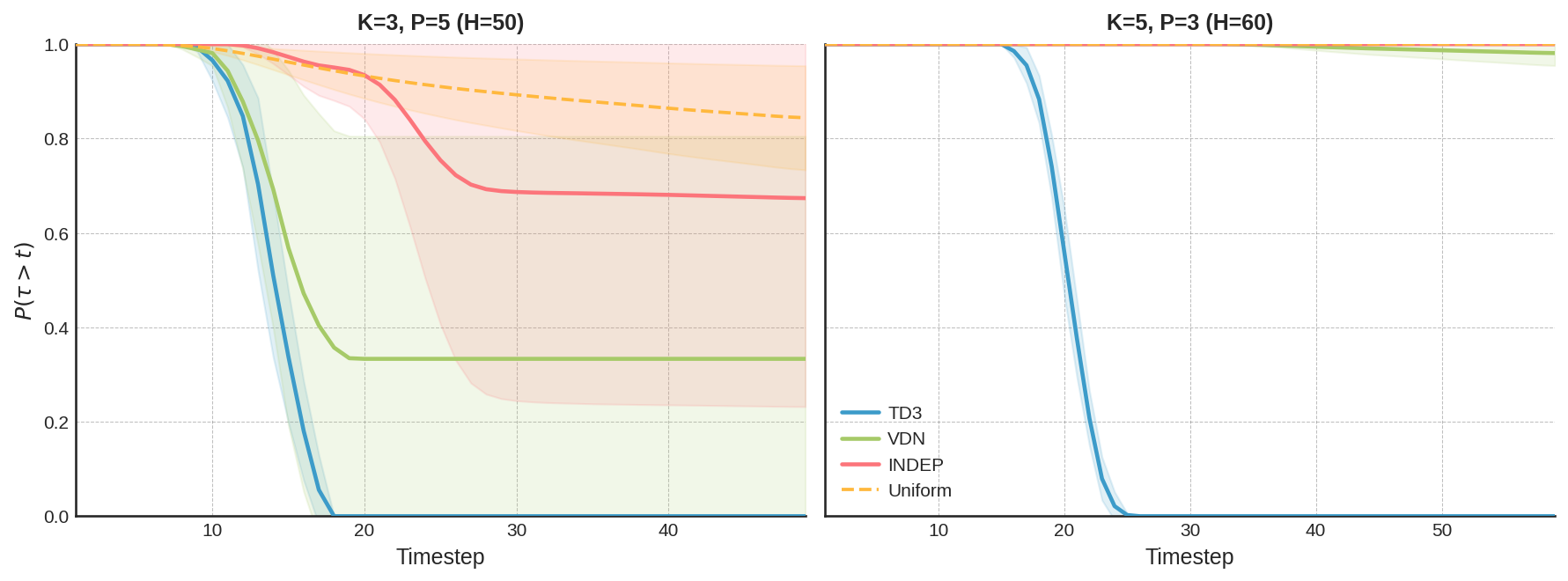}
    \caption{Survival plots for the bandit environments. The survival probability measures the fraction of runs that have not yet selected the stopping action at each timestep. Curves are averaged over three random seeds for each method.}
    \label{fig:survival_bandit}
\end{figure}

\subsubsection{Model Robustness.}
\label{paragraph:model_robustness_bandit}

We evaluate the robustness of our framework under different prior distributions over the environment parameters. In the default setting, both arm capacities and arm means are sampled uniformly. To test robustness with respect to the arm-mean prior, we replace the uniform prior with a scaled Beta distribution parameterized by $(\alpha,\beta)$ and supported on $[0.3,1]$. When $\alpha > \beta$, the distribution places more mass on larger arm means, resulting in environments with more informative feedback and consequently higher inference accuracy.

Figure~\ref{fig:beta_heatmap} reports the average inference accuracy and average stopping time under different Beta priors. Row~(a) corresponds to the A3P5 setting, while Row~(b) corresponds to the A5P3 setting. In each row, the left panel shows inference accuracy and the right panel shows the average stopping time.

\begin{figure}[ht]
    \centering

    \begin{subfigure}[t]{0.45\linewidth}
        \centering
        \includegraphics[width=\linewidth]{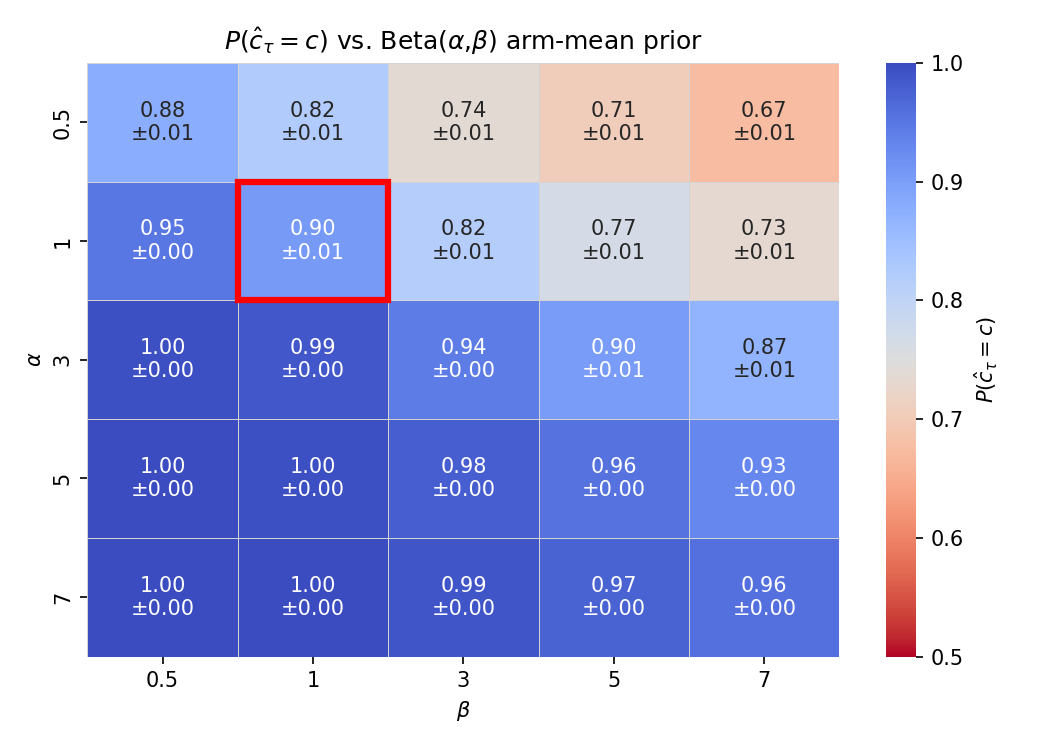}
        \caption{Inference accuracy in A3P5.}
    \end{subfigure}
    \hfill
    \begin{subfigure}[t]{0.45\linewidth}
        \centering
        \includegraphics[width=\linewidth]{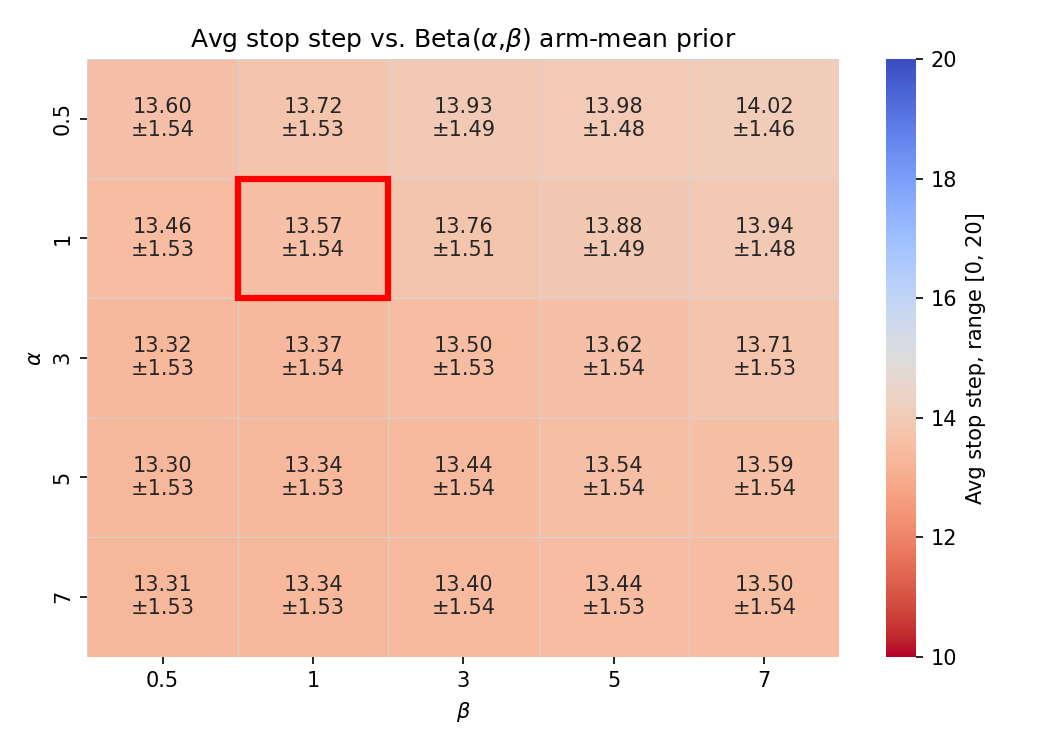}
        \caption{Average stopping time in A3P5.}
    \end{subfigure}

    \vspace{0.4cm}

    \begin{subfigure}[t]{0.45\linewidth}
        \centering
        \includegraphics[width=\linewidth]{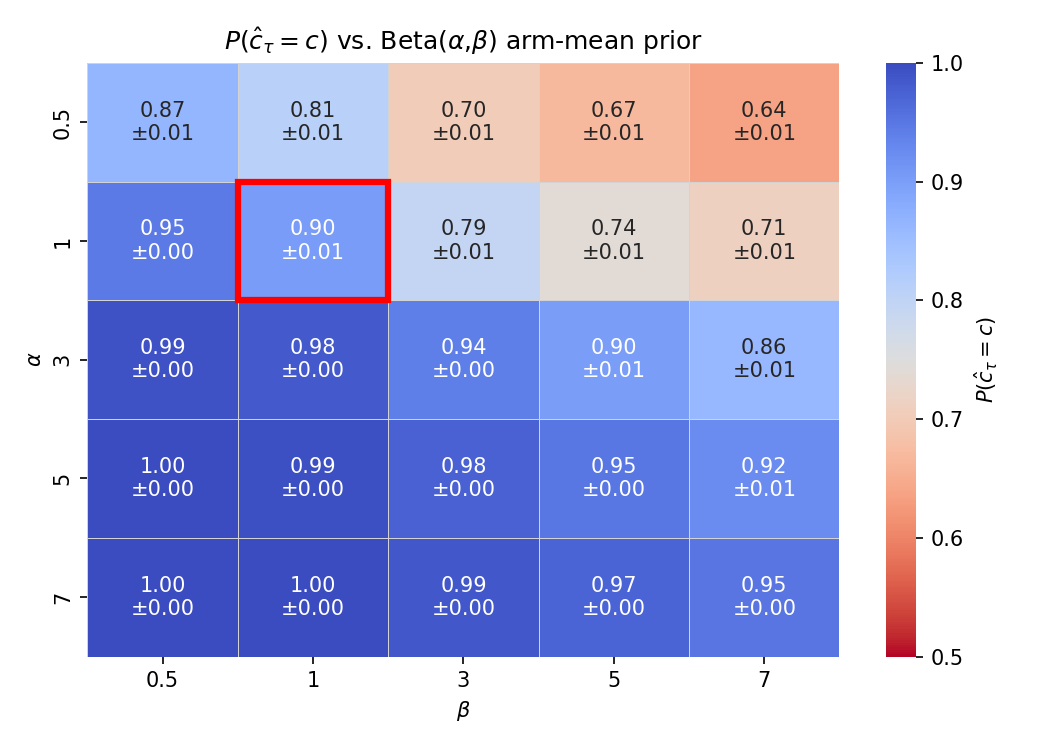}
        \caption{Inference accuracy in A5P3.}
    \end{subfigure}
    \hfill
    \begin{subfigure}[t]{0.45\linewidth}
        \centering
        \includegraphics[width=\linewidth]{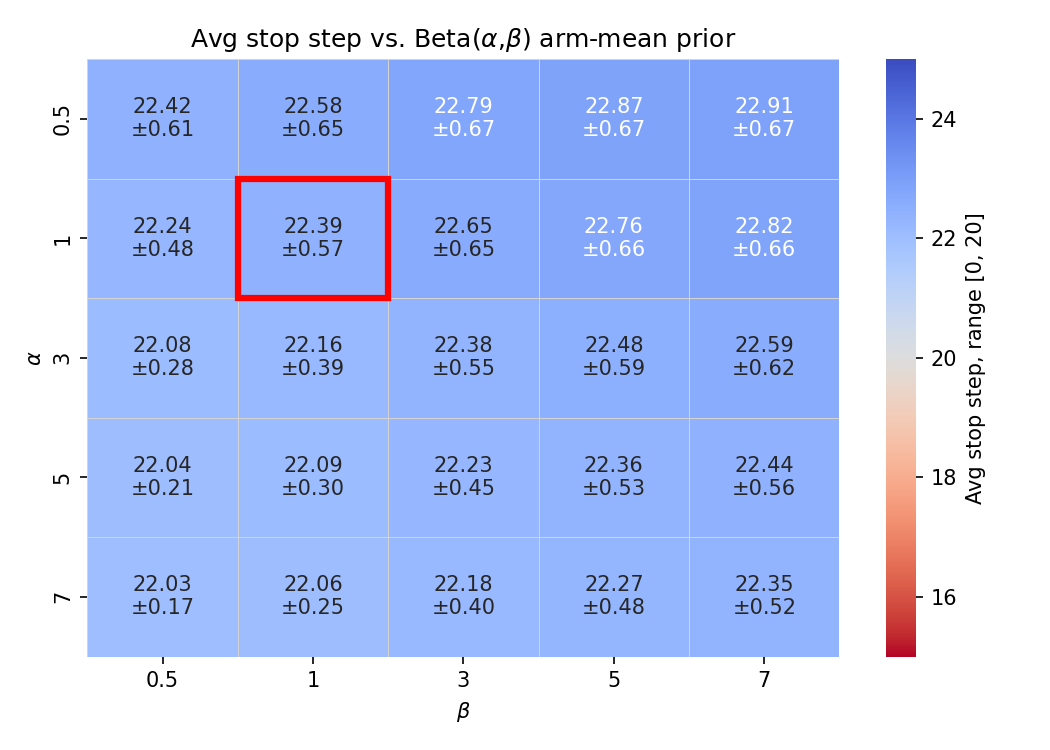}
        \caption{Average stopping time in A5P3.}
    \end{subfigure}

    \caption{Robustness analysis under different scaled Beta priors over arm means. Panels (a) and (c) show inference accuracy, while panels (b) and (d) show average stopping time. The top row corresponds to the A3P5 setting and the bottom row corresponds to the A5P3 setting.}
    \label{fig:beta_heatmap}
\end{figure}
We additionally study robustness with respect to the arm-capacity prior. Instead of sampling capacities uniformly, we generate a probability vector over capacities using a symmetric Dirichlet distribution,
\[
p \sim \mathrm{Dir}(\theta \mathbf{1}_K),
\]
and then sample each arm capacity according to the resulting distribution. We consider $\theta \in \{0.1, 0.5, 1, 5, 10\}$, where smaller values of $\theta$ produce more skewed capacity distributions, while larger values approach a uniform distribution. Figure~\ref{fig:dirichlet_robustness} summarizes the resulting inference accuracy under different Dirichlet priors.

\begin{figure}[ht]
    \centering

    \includegraphics[width=0.45\linewidth]{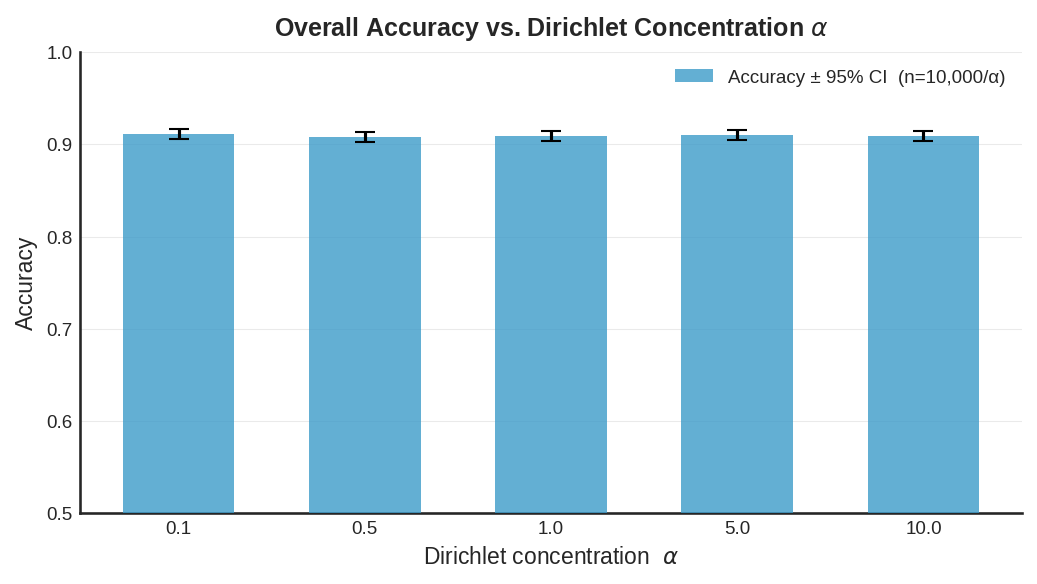}
    \hfill
    \includegraphics[width=0.45\linewidth]{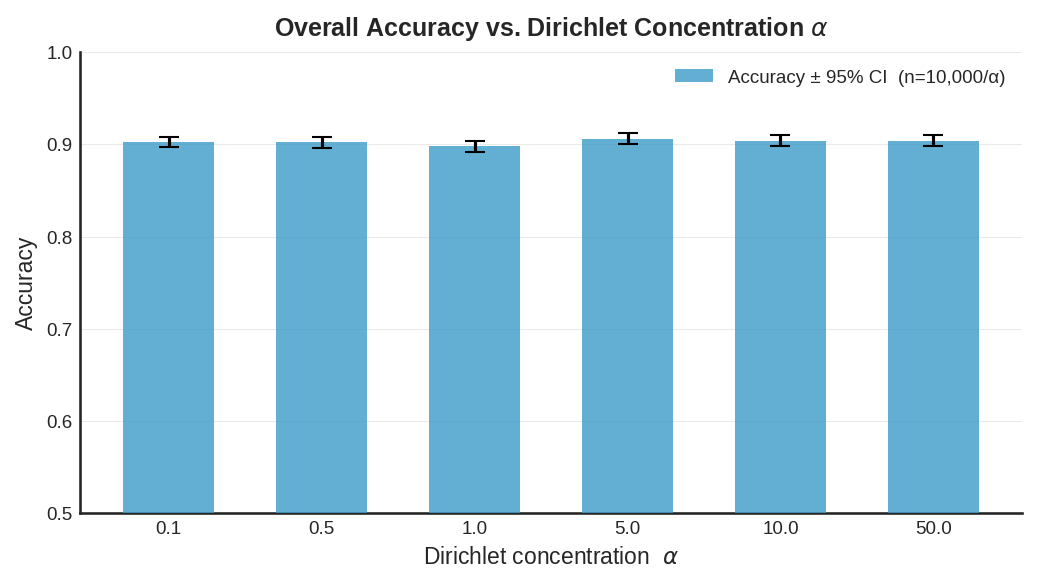}

    \caption{Robustness analysis under different Dirichlet priors over arm capacities. Left: A3P5 setting. Right: A5P3 setting.}
    \label{fig:dirichlet_robustness}
\end{figure}
\clearpage

\subsubsection{Coordination Strategy}
\begin{table}[htbp]
\centering
\small
\caption{Most common joint exploration action at each step across 10{,}000
sampled environments (A3P5) under the greedy ICMAPE-TD3 policy.
At step $t$, only environments still active (stopping time $\geq t$) are
counted ($n_{\text{active}}$).
The joint action $(a^{(1)}, a^{(2)}, a^{(3)}, a^{(4)}, a^{(5)})$ lists the
arm index (0--2) pulled by each player in the modal joint action.
Steps beyond $t{=}15$ are omitted as all environments have stopped.}
\label{tab:joint_action_freq}
\begin{tabular}{clrrc}
\toprule
$t$ & Joint action & Count & $n_{\text{active}}$ & \% \\
\midrule
 0 & $(2,2,2,2,2)$ & 10{,}000 & 10{,}000 & 100.0 \\
 1 & $(2,2,2,2,2)$ & 10{,}000 & 10{,}000 & 100.0 \\
 2 & $(2,2,2,2,2)$ & 10{,}000 & 10{,}000 & 100.0 \\
 3 & $(2,0,2,2,2)$ & 10{,}000 & 10{,}000 & 100.0 \\
 4 & $(2,0,0,2,2)$ & 10{,}000 & 10{,}000 & 100.0 \\
 5 & $(0,0,0,0,0)$ & 10{,}000 & 10{,}000 & 100.0 \\
 6 & $(0,0,0,0,1)$ & 10{,}000 & 10{,}000 & 100.0 \\
 7 & $(1,0,0,0,1)$ & 10{,}000 & 10{,}000 & 100.0 \\
 8 & $(1,0,0,1,1)$ & 10{,}000 & 10{,}000 & 100.0 \\
 9 & $(1,1,1,1,1)$ & 10{,}000 & 10{,}000 & 100.0 \\
10 & $(1,2,1,1,1)$ &  9{,}031 & 10{,}000 &  90.3 \\
11 & $(1,2,2,1,1)$ &  6{,}279 &  9{,}031 &  69.5 \\
12 & $(1,2,2,1,1)$ &  4{,}510 &  6{,}279 &  71.8 \\
13 & $(1,2,2,2,1)$ &  4{,}236 &  5{,}752 &  73.6 \\
14 & $(0,0,0,0,0)$ &  3{,}872 &  4{,}236 &  91.4 \\
15 & $(0,0,0,0,0)$ &    364   &    364   & 100.0 \\
\bottomrule
\end{tabular}
\end{table}

\begin{table}[htbp]
\centering
\small
\caption{Frequency of joint stop-action combinations across 10{,}000 sampled
environments (A3P5).
A combination is recorded when one or more players simultaneously select the stop action at the terminal step.
\emph{\% vol.} is the fraction of voluntary-stop episodes;
\emph{\% all} is the fraction of all 10{,}000 environments.}
\label{tab:stop_combinations}
\begin{tabular}{lrcc}
\toprule
Combination & Count & \% vol.\ stops & \% all envs \\
\midrule
\multicolumn{4}{l}{\textit{1 player stopping simultaneously}} \\
\quad $\{P_5\}$                         & 1{,}778 & 17.8 & 17.8 \\
\quad $\{P_4\}$                         & 1{,}056 & 10.6 & 10.6 \\
\quad $\{P_1\}$                         &   947   &  9.5 &  9.5 \\
\quad $\{P_3\}$                         &    10   &  0.1 &  0.1 \\
\midrule
\multicolumn{4}{l}{\textit{2 players stopping simultaneously}} \\
\quad $\{P_1, P_5\}$                    & 2{,}723 & 27.2 & 27.2 \\
\quad $\{P_4, P_5\}$                    &   813   &  8.1 &  8.1 \\
\quad $\{P_1, P_4\}$                    &   382   &  3.8 &  3.8 \\
\quad $\{P_3, P_5\}$                    &    18   &  0.2 &  0.2 \\
\quad $\{P_2, P_4\}$                    &    14   &  0.1 &  0.1 \\
\quad $\{P_1, P_3\}$                    &     2   &  0.0 &  0.0 \\
\quad $\{P_3, P_4\}$                    &     1   &  0.0 &  0.0 \\
\midrule
\multicolumn{4}{l}{\textit{3 players stopping simultaneously}} \\
\quad $\{P_1, P_4, P_5\}$              & 1{,}806 & 18.1 & 18.1 \\
\quad $\{P_1, P_3, P_5\}$              &   125   &  1.2 &  1.2 \\
\quad $\{P_2, P_4, P_5\}$              &    25   &  0.2 &  0.2 \\
\quad $\{P_3, P_4, P_5\}$              &     3   &  0.0 &  0.0 \\
\quad $\{P_1, P_2, P_4\}$              &     1   &  0.0 &  0.0 \\
\quad $\{P_1, P_3, P_4\}$              &     1   &  0.0 &  0.0 \\
\midrule
\multicolumn{4}{l}{\textit{4 players stopping simultaneously}} \\
\quad $\{P_1, P_3, P_4, P_5\}$        &   257   &  2.6 &  2.6 \\
\quad $\{P_1, P_2, P_4, P_5\}$        &     5   &  0.1 &  0.1 \\
\quad $\{P_2, P_3, P_4, P_5\}$        &     1   &  0.0 &  0.0 \\
\quad $\{P_1, P_2, P_3, P_5\}$        &     1   &  0.0 &  0.0 \\
\midrule
\multicolumn{4}{l}{\textit{5 players stopping simultaneously}} \\
\quad $\{P_1, P_2, P_3, P_4, P_5\}$   &    31   &  0.3 &  0.3 \\
\bottomrule
\end{tabular}
\end{table}

\clearpage

\subsection{Signal Source Localization}

We provide additional results for the signal source localization environment, complementing the final performance results reported in the main text.

\begin{figure}[ht]
\centering
\includegraphics[width=0.98\linewidth]{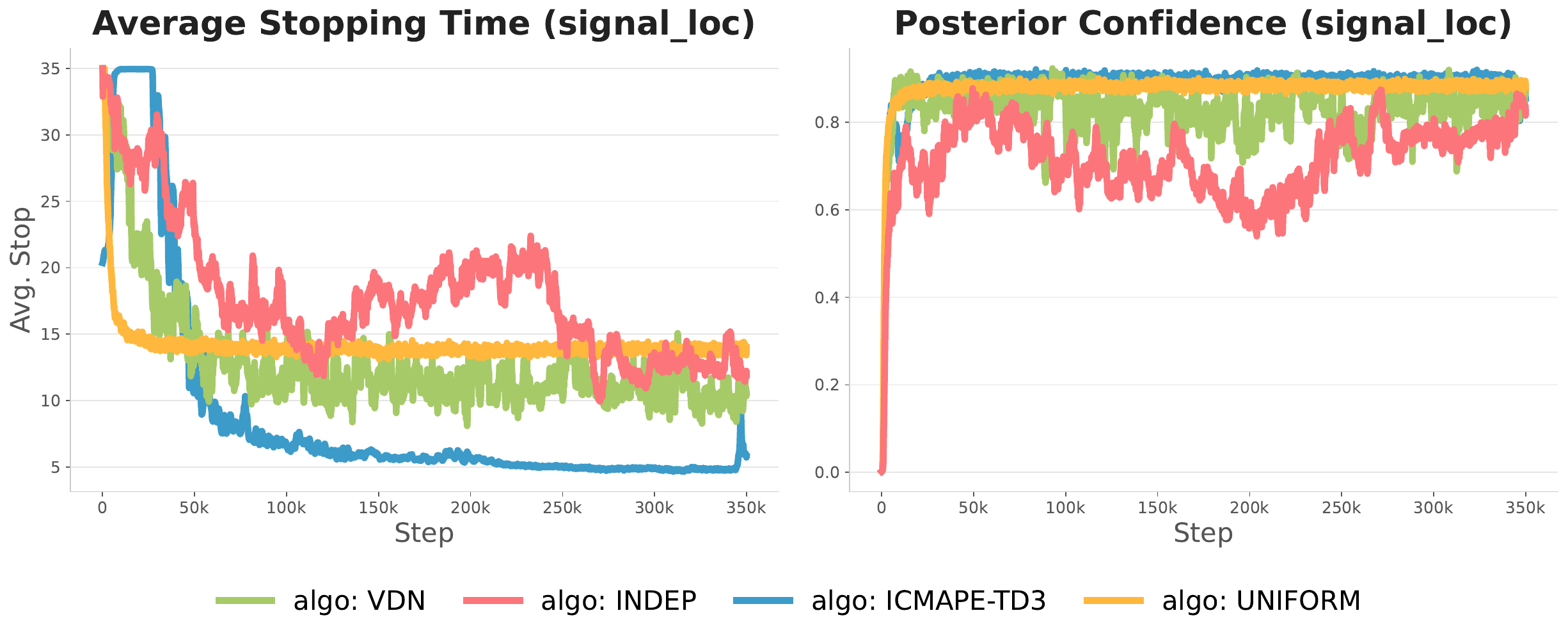}
\caption{Signal Source Localization Training Dynamics. Curves are averaged over three random seeds and smoothed with an exponential moving average parameter of $0.2$.}
\label{fig:training_dynamics_signal}
\end{figure}

We further analyze the stopping behavior of different methods in the signal source localization environment using survival plots (Figure~\ref{fig:survival_robot}). 
\begin{figure}[ht]
    \centering
    \includegraphics[width=0.5\linewidth]{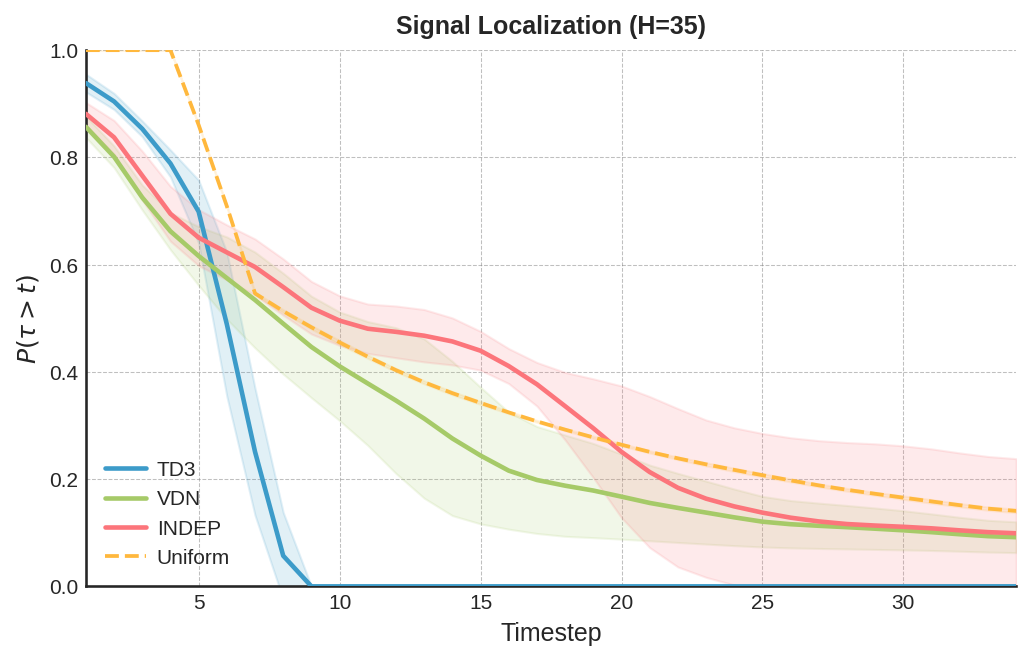}
    \caption{Survival plots for the signal source localization environment. The survival probability measures the fraction of runs that have not yet selected the stopping action at each timestep. Curves are averaged over three random seeds for each method.}
    \label{fig:survival_robot}
\end{figure}

\paragraph{Model Robustness.}

We also evaluate model robustness in the signal source localization environment. In the default setting, the minimum Manhattan distance between the true signal source and the distractor is set to 10, and source distractor pairs are sampled uniformly at random. To study the effect of source distractor separation, we fix the Manhattan distance between the two points to values ranging from 10 to 48. For each distance value, we sample 10,000 instances and report the resulting inference accuracy in Figure~\ref{fig:dist_acc}.

\begin{figure}[ht]
    \centering
    \includegraphics[width=0.5\linewidth]{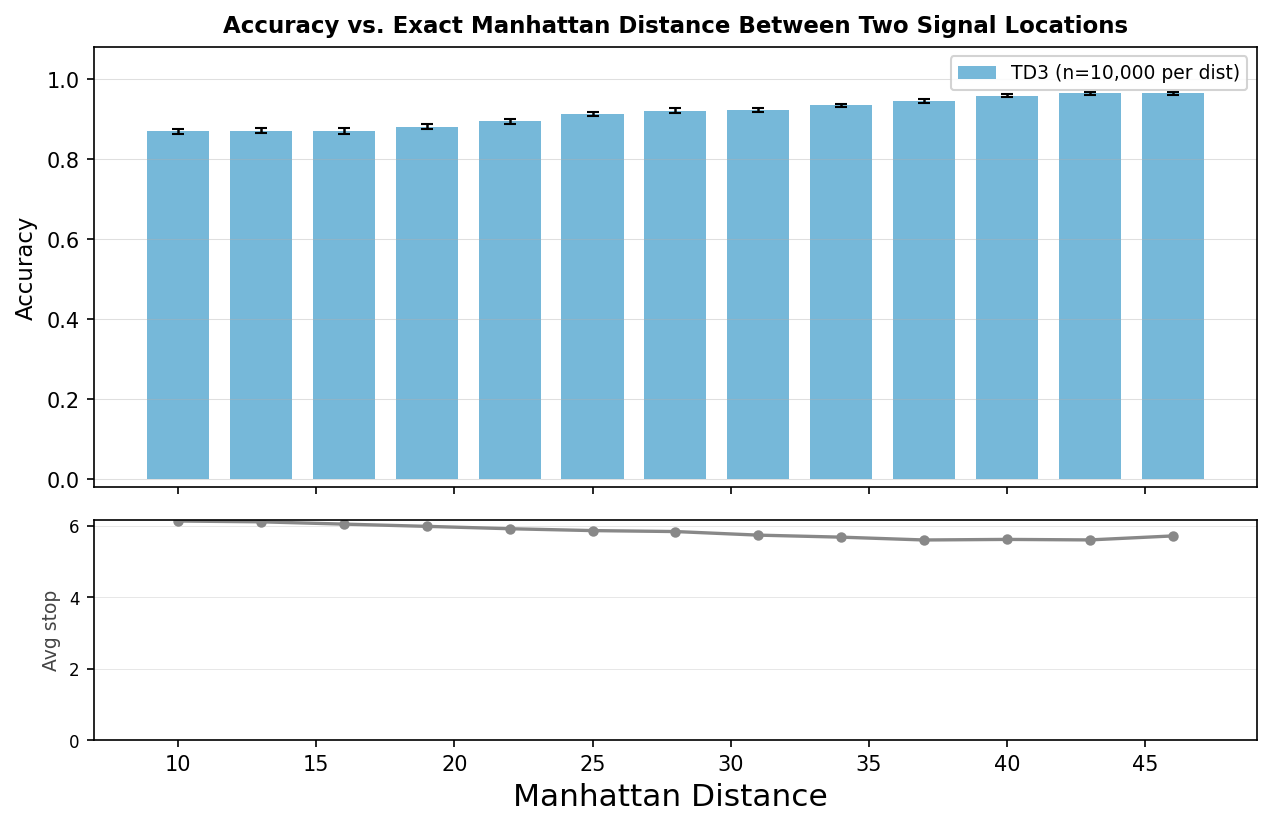}
    \caption{Inference accuracy as a function of the Manhattan distance between the true signal source and the distractor. Larger distances generally make the two sources easier to distinguish. Results are computed over 10,000 sampled instances for each distance value.}
    \label{fig:dist_acc}
\end{figure}

Similar to the robustness analysis over arm means in the capacity identification environment, we also test robustness with respect to the distractor intensity distribution. Specifically, we sample the distractor intensity from a scaled Beta distribution supported on $[0,0.7]$ with different $(\alpha,\beta)$ parameters. Figure~\ref{fig:signal_intensity_beta} reports the inference accuracy under these different distractor-intensity priors.

\begin{figure}[ht]
    \centering
    \includegraphics[width=0.5\linewidth]{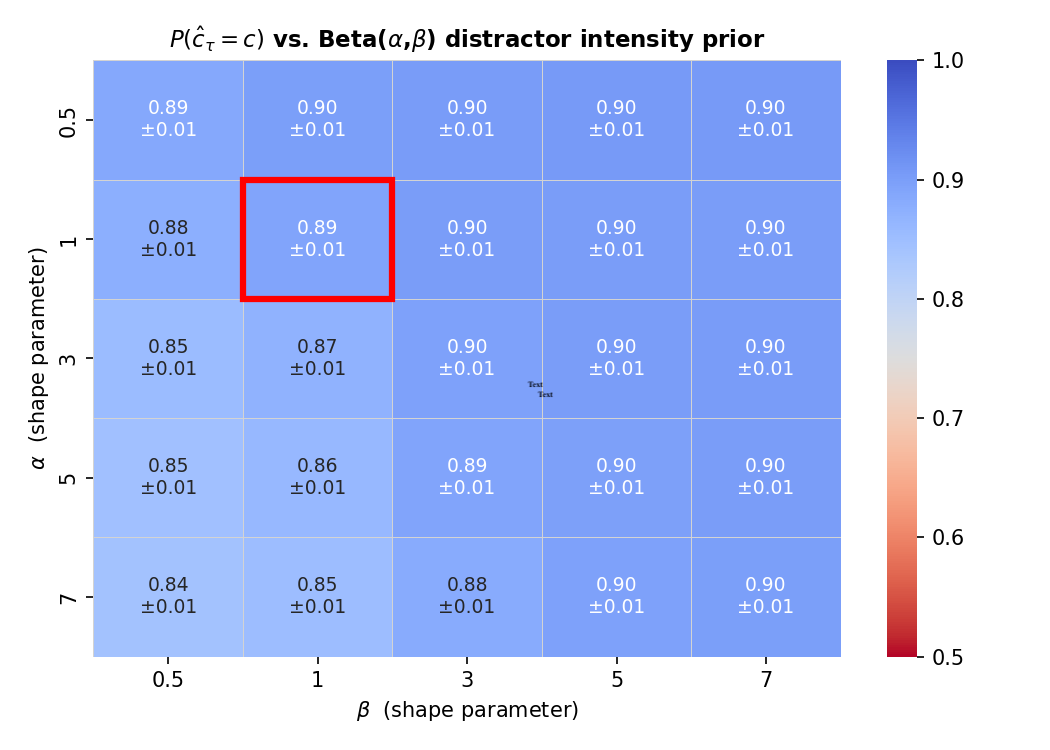}
    \caption{Robustness analysis under different scaled Beta priors over distractor intensity in the signal source localization environment. The heatmap reports inference accuracy for different choices of $(\alpha,\beta)$, with distractor intensities scaled to lie in $[0,0.7]$.}
    \label{fig:signal_intensity_beta}
\end{figure}
\FloatBarrier

\subsubsection{Coordination Strategy}
\begin{table}[htbp]
\centering
\small
\caption{Most common action for each agent at each step across active environments under the greedy ICMAPE-TD3 policy in the signal localization
(10{,}000 samples).
Each cell in the table reports the agent action and its frequency among the
$n_{\text{active}}$ environments still running at that step.
Action labels: \textbf{up} $=$ up \textbf{dn} $=$ down,
\textbf{lf} $=$ left, \textbf{rt} $=$ right; \textbf{STOP} $=$ stop action.
Agent 1 and agent 2 maintains a fully deterministic trajectory throughout}
\label{tab:robot_per_player_action_freq}
\begin{tabular}{ccccc}
\toprule
$t$ & Agent 1 & Agent 2 & Agent 3 & $n_{\text{active}}$ \\
\midrule
0 & dn (100.0\%) & dn (100.0\%) & rt (\phantom{0}99.9\%) & 10{,}000 \\
1 & dn (100.0\%) & dn (100.0\%) & rt (\phantom{0}94.6\%) &  9{,}994 \\
2 & rt (100.0\%) & dn (100.0\%) & rt (\phantom{0}92.6\%) &  9{,}459 \\
3 & up (100.0\%) & dn (100.0\%) & up (\phantom{0}95.2\%) &  8{,}759 \\
4 & up (100.0\%) & dn (100.0\%) & up (\phantom{0}93.0\%) &  8{,}338 \\
5 & up (100.0\%) & dn (100.0\%) & up (\phantom{0}87.5\%) &  7{,}755 \\
6 & rt (100.0\%) & dn (100.0\%) & lf (\phantom{0}86.0\%) &  6{,}783 \\
7 & rt (100.0\%) & \textbf{STOP} (100.0\%) & up (\phantom{0}87.3\%) &  5{,}835 \\
\bottomrule
\end{tabular}
\end{table}

\begin{table}[htbp]
\centering
\small
\caption{Frequency of joint stop-action combinations across 10{,}000 sampled
environments in the robot wildfire task (Signal Localization).
\emph{\% vol.} is the fraction of voluntary-stop episodes;
\emph{\% all} is the fraction of all 10{,}000 environments.}
\label{tab:robot_stop_combinations}
\begin{tabular}{lrcc}
\toprule
Combination & Count & \% vol.\ stops & \% all envs \\
\midrule
\multicolumn{4}{l}{\textit{1 robot stopping simultaneously}} \\
\quad $\{P_2\}$ & 5{,}095 & 50.9 & 50.9 \\
\quad $\{P_3\}$ & 4{,}165 & 41.6 & 41.6 \\
\midrule
\multicolumn{4}{l}{\textit{2 robots stopping simultaneously}} \\
\quad $\{P_2, P_3\}$ & 740 & 7.4 & 7.4 \\
\bottomrule
\end{tabular}
\end{table}

\FloatBarrier

\section{Limitations and Broader Impact}
\label{sec:limitations_broader_impact}
Several limitations suggest directions for future work. Our current framework is trained using an off-policy actor--critic method. While this provides sample reuse and stable value learning, future work could explore on-policy alternatives such as PPO \cite{schulman2017ppo} to better understand how different reinforcement learning objectives affect decentralized information acquisition and coordination strategy in our setting. Another limitation is that our experiments focus on fixed hypothesis spaces, where the underlying hypothesis remains constant throughout an episode. A possible direction is to study more complex environments in which the hypothesis evolves over time, requiring agents to continually adapt their exploration strategies. More broadly, our work studies hypothesis identification from a prior distribution over environments. Future research could extend this setting toward more general forms of environment understanding, where agents are not restricted to identifying a pre-specified hypothesis, but instead learn richer representations of the sampled environment that support a wider range of downstream inference tasks.

\end{document}